\documentclass[letterpaper,11pt]{article}

\PassOptionsToPackage{numbers,sort,compress}{natbib}
\usepackage[hyperref,abbrvbib]{jmlr_adapted}
\usepackage{microtype}
\usepackage{graphicx}
\usepackage{subfigure}
\usepackage{booktabs} 

\usepackage{mathtools}
\usepackage{amssymb}
\usepackage{amsthm}
\usepackage{amsfonts} 
\usepackage{bbm}
\usepackage[english]{babel}
\usepackage[linesnumbered,ruled,vlined]{algorithm2e}
\usepackage[noabbrev,capitalize]{cleveref}
\usepackage[flushleft]{threeparttable}
\usepackage{url}
\usepackage{upgreek}  
\usepackage{xcolor}

\usepackage{pifont} 
\newcommand{\cmark}{\ding{51}}%

\newcommand{\E}{\mathbb{E}}

\newcommand{\real}{\mathbb{R}}

\newcommand{\states}{\mathcal{S}}
\newcommand{\actions}{\mathcal{A}}
\newcommand{\opt}{^\star}

\DeclareMathOperator*{\argmax}{argmax}
\DeclareMathOperator{\KL}{KL}
\renewcommand{\exp}[1]{\operatorname{exp}\left( #1\right) }

\newcommand{\erm}{\operatorname{ERM}}
\newcommand{\var}{\operatorname{VaR}}
\newcommand{\cvar}{\operatorname{CVaR}}
\newcommand{\evar}{\operatorname{EVaR}}
\newcommand{\vari}{\operatorname{var}}

\newcommand{\Real}{\mathbb{R}}

\renewcommand{\cite}[1]{\citep{#1}}

\title{RASR: Risk-Averse Soft-Robust MDPs\\with EVaR and Entropic Risk}

\author{%
  Jia Lin Hau \email jialin.hau@unh.edu \\
  University of New Hampshire \AND
  Marek Petrik \email mpetrik@cs.unh.edu \\
  University of New Hampshire  \AND
  Mohammad Ghavamzadeh \email ghavamza@google.com \\ 
  Google Research
  \AND
  Reazul Russel \email rz.hasan23@gmail.com \\
  Facebook%
}
\date{}
\newcommand{\journal}[1]{}

\begin{document}
\maketitle

\begin{abstract}
Prior work on safe Reinforcement Learning (RL) has studied risk-aversion to randomness in dynamics (aleatory) and to model uncertainty (epistemic) in isolation. We propose and analyze a new framework to jointly model the risk associated with epistemic and aleatory uncertainties in finite-horizon and discounted infinite-horizon MDPs. We call this framework that combines Risk-Averse and Soft-Robust methods RASR. We show that when the risk-aversion is defined using either EVaR or the entropic risk, the optimal policy in RASR can be computed efficiently using a new dynamic program formulation with a time-dependent risk level. As a result, the optimal risk-averse policies are deterministic but time-dependent, even in the infinite-horizon discounted setting. We also show that particular RASR objectives reduce to risk-averse RL with mean posterior transition probabilities. Our empirical results show that our new algorithms consistently mitigate uncertainty as measured by EVaR and other standard risk measures. 
\end{abstract}

\section{Introduction}
\label{sec:intro}

A major concern in high-stakes applications of reinforcement learning~(RL), such as those in healthcare and finance, is to quantify the risk associated with the variability of returns. This variability is a form of \emph{aleatory} uncertainty that arises from the inherent randomness in system dynamics. Since the risk of random returns cannot be captured by the standard \emph{expected} objective, \emph{convex risk measures} have emerged as perhaps the most popular tools to quantify this risk in RL and beyond. They are sufficiently general to capture a wide range of stakeholder preferences and are more computationally convenient than many other alternatives~\cite{Follmer2004}. Conditional value-at-risk (CVaR), entropic value-at-risk (EVaR)~\cite{Ahmadi-Javid2012,Follmer2011}, and entropic risk measure (ERM)~\cite{Follmer2004} are common examples of convex risk measures. 

The goal in robust Markov decision process (MDP) is to mitigate performance loss due to uncertainty in modeling the system dynamics~\cite{Grand-Clement2021b,Iyengar2005,Ho2021a}. This uncertainty, often caused by limited or noisy data, is a form of \emph{epistemic} uncertainty. \emph{Soft-robust} formulations refine robust optimization by assuming a Bayesian distribution over plausible models (of the system dynamics) and then quantify the risk of model errors using convex risk measures~\cite{Derman2018,Lobo2021}. These formulations have close connections to distributional robustness~\cite{Xu2012}. While being risk-averse to epistemic uncertainty, existing soft-robust RL formulations are risk-neutral when it comes to the aleatory uncertainty  that arises from the randomness in the system dynamics. This combination of risk-aversion to epistemic uncertainty with risk-neutrality to aleatory uncertainty can be problematic from the modeling perspective~\cite{Chow2018}, and as we show below, may introduce unnecessary computational complexity.

The overarching objective of our work is to compute policies for MDPs that \emph{jointly} mitigate the risk associated with epistemic (model) and aleatory (random dynamics) uncertainties. We call this objective RASR as it combines Risk Averse (aleatory) and Soft-Robust (epistemic) methods. This is in contrast to the existing soft-robust MDP algorithms that are risk-neutral to the aleatory uncertainty. In this paper, we study RASR with two popular risk measures: ERM and EVaR. 

As our first contribution, in \cref{sec:rasr-erm}, we introduce our RASR-ERM framework and propose new dynamic programming algorithms and analysis for it. ERM is unique among law-invariant risk measures in being dynamically consistent~\cite{Kupper2006}, which makes it compatible with dynamic programming (DP). Unfortunately, ERM is \emph{not} positively homogeneous, which makes it incompatible with the use of discount factors. As a result, ERM has only been solved exactly in average-reward MDPs~\cite{Borkar2002a} and \emph{undiscounted} stochastic programs~\cite{Dowson2021}. Our main innovation is to use time-dependent risk levels to precisely solve ERM in \emph{discounted finite-horizon} MDPs and to employ new bounds to tightly approximate it in \emph{discounted infinite-horizon} MDPs (\cref{subsec:VI-RASR-ERM}). We build on the DP decomposition of the RASR-ERM objective to show that there exists an optimal value function and (surprisingly) a \emph{deterministic} Markov optimal policy for this problem. This is unusual because most other risk-averse formulations require randomized optimal policies. We also show that under an assumption of a dynamic model of epistemic uncertainty~\cite{Derman2018,Eriksson2020}, the RASR-ERM objective reduces to a risk-averse MDP with the mean posterior transition model (\cref{subsec:DP-RASR-ERM}).

As our second contribution, we formulate and study the RASR-EVaR framework in \cref{sec:rasr-evar}. Although ERM is computationally convenient, it is often an impractical method to measure risk since the result is scale-dependent. EVaR is preferable to ERM because it is coherent, positively-homogenous, interpretable, and comparable with VaR and CVaR. However, EVaR is not dynamically consistent and cannot be directly optimized using a DP. Our main contribution here is to reduce the RASR-EVaR optimization to multiple RASR-ERM problems that each can be solved by DP. Our theoretical analysis shows that the RASR-EVaR properties mirror those for RASR-ERM and that the proposed algorithm can compute a solution arbitrarily close to the optimum. We empirically evaluate our RASR algorithms in \cref{sec:empirical} and show their benefits over prior robust, soft-robust, and risk-averse MDP algorithms. Finally, in \cref{sec:related_work}, we position our RASR framework in the context of the literature on soft-robust and risk-averse MDPs.

\section{Preliminaries} 
\label{sec:preliminaries}

We assume the decision problem can be formulated as an MDP, defined by the tuple $(\states, \actions, r, p, s_0,\gamma)$. The state and action sets  $\states$ and $\actions$ are finite with cardinality $S$ and $A$. The reward function $r\colon \states \times \actions \to \Real$ represents the reward received in each state after taking an action. We use $\triangle r = \max_{s\in \states, a\in \actions} r(s,a) - \min_{s\in \states, a\in \actions} r(s,a)$ to refer to the span semi-norm of the rewards. The transition probabilities are denoted as $p\colon\states \times \actions \to \Delta^S$, where $\Delta^S$ is the probability simplex in $\Real^S$. The initial state is denoted by $s_0\in \states$. Finally, $\gamma \in (0,1]$ is the discount factor. We assume a fixed horizon $T \in \mathbb{N}^+ \cup \{\infty\}$ with $T = \infty$ indicating an infinite-horizon objective. While a discounted finite-horizon objective is uncommon in practice, it serves us as an intermediate step when analyzing the infinite-horizon objective.

The most-general solution to an MDP is a randomized history-dependent policy that at each time step prescribes a distribution over actions as a function of the history up to that step~\cite{Puterman2005}. A \emph{randomized Markov} policy depends only on the time step $t$ and current state $s_t$ as $\pi=(\pi_t)_{t=0}^{T-1}$, where $\pi_t\colon \states \to \Delta^A$. A policy $\pi$ is \emph{stationary} when it is time-independent (all $\pi_t$'s are equal), in which case we omit the time subscript. We denote by $\Pi_{MR}$ and $\Pi_{SR}$, the sets of Markov and stationary randomized policies, and by $\Pi_{MD}$ and $\Pi_{SD}$, the corresponding sets of deterministic policies. The set of history-dependent randomized policies is denoted by $\Pi_{HR}$

We define $\mathfrak{R}_T^\pi$, the random variable of the return of a policy $\pi$ after $T$ time steps as
\begin{equation} 
\label{eq:return}
    \mathfrak{R}_T^\pi \;=\; \sum_{t=0}^{T-1} \gamma^t \cdot R_t^\pi = \sum_{t=0}^{T-1} \gamma^t \cdot r(S_t^{\pi}, A^{\pi}_t)~, 
\end{equation}
where $S_t^\pi$, $A_t^\pi \sim \pi_t(\cdot|S_t^{\pi})$, and $R_t^{\pi}$ are the random variables of state visited, action taken, and reward received respectively at time $t\in 0, \dots , T$, when following policy $\pi$. The objective in the standard risk-neutral MDP is to maximize the \emph{expectation} of the return random variable, 
\begin{equation}
\label{eq:risk-neutral-rl}
\max_{\pi \in \Pi_{{MR}}} \, \E \big[  \mathfrak{R}^{\pi}_T \big] ~.
\end{equation}
In finite-horizon MDPs, $T < \infty$ and (usually) $\gamma = 1$. In infinite-horizon discounted MDPs, we use $T = \infty$ (as a shorthand for $T\rightarrow\infty$) and restrict the discount factor to $\gamma\in(0,1)$. It is known that the finite and infinite horizon discounted settings have optimal policies in $\Pi_{MD}$ and $\Pi_{SD}$, respectively. 

\paragraph{Risk-averse MDP} 
A risk measure $\psi\colon \mathbb{X} \to \mathbb{R}$ assigns a scalar risk value to a random variable $X\in\mathbb{X}$, where $\mathbb{X}$ denotes the set of real-valued random variables. Convex risk measures are an axiomatic generalization of the expectation operator $\E[\cdot]$ that capture a wide range of risk-aversion preferences~\cite{Follmer2002,Frittelli2002}. We describe \emph{coherent} and \emph{convex} risk measures, and summarize their properties that are desirable in studying risk-averse MDPs in \cref{secapp:risk-measure-prop}. The objective in risk-averse MDP is defined by replacing the expectation in~\eqref{eq:risk-neutral-rl} with an appropriate risk measure  
\begin{equation} \label{eq:risk-averse-rl-ent}
  \max_{\pi\in \Pi_{HR}} \, \psi\big[\mathfrak{R}_{T}^\pi\big]. 
\end{equation}

\paragraph{Soft-robust MDP} 
The soft-robust setting makes the Bayesian assumption that the transition model $P$ is a random variable with a distribution that can be computed, for instance, using Bayesian inference~\cite{Derman2018,Eriksson2020,Lobo2021}. In this paper, we assume a \emph{dynamic} model of uncertainty~\cite{Derman2018,Eriksson2020}. In the dynamic model, the transition probability is not only unknown, but can also change during the execution. This is in contrast to the \emph{static} model~\cite{Delage2009,Lobo2021}, in which it is uncertain but does not change throughout an episode. We target dynamic uncertainty because it is easier to optimize and our results lay down the foundations necessary to tackle static models in future. In the dynamic model, the transition probability is defined as $P = (P_t)_{t=0}^{T-1}$, where each model $P_t \colon \Omega \to (\Delta^{S})^{\mathcal{S}\times \mathcal{A}}$ is a random variable that represent the uncertain transition functions for some finite set of possible models $\Omega$. The uncertain models are distributed independently across time as $P_t \sim f_t$, and $f_t \in \Delta^{\Omega }, \, t=0, \dots , T-1$ are derived from Bayesian inference methods. Note that the upper-case $P$ is a random variable that represents an uncertain transition function in a soft-robust MDP, whereas the lower-case $p$ is used to denote a known transition function in an MDP.

Prior work on soft-robust RL (e.g.,~\cite{Derman2018,Eriksson2020,Lobo2021}) has focused on the following objective:
\begin{equation} 
\label{eq:soft_robust_return}
\max_{\pi \in \Pi_{HR}} \; \psi\Big[ \mathop{\E} \big[\mathfrak{R}_{T}^\pi \mid P \big] \Big].
\end{equation}
In~\eqref{eq:soft_robust_return}, the risk measure $\psi$ is applied only to the epistemic uncertainty over $P$, and the optimization is risk-neutral (uses $\E[\cdot]$) to the randomness in $\mathfrak{R}_{T}^\pi \mid P$ (aleatory uncertainty). For some particular choices of $\psi$, the optimization in~\eqref{eq:soft_robust_return} reduces to a distributionally-robust MDP~\cite{Xu2012,Grand-Clement2021b,Lobo2021}. 

\paragraph{RASR}
Our RASR formulation, introduced formally below, takes into account both the epistemic uncertainty in the transition model $P$ and the aleatory uncertainty in $\mathfrak{R}_{T}^\pi \mid P$, and optimizes the objective
\begin{equation} 
\label{eq:RASR}
\max_{\pi \in \Pi_{HR}} \; \psi\Big[ \psi \big[\mathfrak{R}_{T}^\pi \mid P \big] \Big].
\end{equation}
Given the well-known equivalence between coherent risk measures and robustness~(e.g.,~\cite{Osogami2012}), it is tempting to conclude that the combination of risk and soft-robustness in~\eqref{eq:RASR} reduces to a robust MDP. Alas, this is false for virtually all reasonable choices of the risk measure $\psi$. The risk-robustness equivalence is only known for risk-averse formulations in~\eqref{eq:soft_robust_return} that admit a dynamic program formulation, that is when $\psi$ is dynamically consistent. As we discuss above, most practical risk measures---such as VaR, CVaR, EVaR and others---are not dynamically consistent~\cite{Iancu2015a}, and therefore, neither~\eqref{eq:risk-averse-rl-ent} nor~\eqref{eq:RASR} readily reduce to a robust MDP.

\paragraph{Risk Measures}
We study two convex risk measures in our RASR formulation: \emph{entropic risk measure} (ERM) andu \emph{entropic value-at-risk} (EVaR). ERM with a risk-aversion parameter $\alpha\in \Real_+ \cup \{\infty\}$, for a random variable $X \in \mathbb{X}$, is defined as~\cite{Follmer2004}
\begin{align} \label{eq:defn_ent_risk}
  \erm^\alpha[X] = - \alpha^{-1} \cdot \log \Bigl( \E \left[e^{-\alpha\cdot  X} \right] \Bigr).
\end{align}
For the risk level $\alpha=0$, ERM of a random variable equals to its expectation,
\(  \erm^{0}[X] = \lim_{\alpha \to 0^{+}} \erm^{\alpha}[X] = \E[X].\) Similarly, $\erm^{\infty}[X] = \operatorname{ess} \inf[X]$ is the minimum value of $X$. Note that we use an ERM definition in \eqref{eq:defn_ent_risk} that is meant to be maximized. ERM definitions designed to be minimized lack the leading negative sign~\cite{Follmer2011}.

ERM is the only law-invariant convex risk measure that is dynamically-consistent~\cite{Kupper2006}~(see \cref{subsec:risk-measure-prop}).
This is an important property for a risk measure in multi-stage decision problems, because it allows defining a dynamic program (DP) for the risk measure and optimizing it. The following theorem is crucial for deriving our results. It has been proved in earlier work, but we report its proof in \cref{app:sec:prelim} for completeness.
\begin{theorem}[Tower Property] \label{thm:tower-erm}
Any two random variables $X_1,X_2\in\mathbb X$ satisfy that
\[\erm^{\alpha}[X_1] \;=\; \erm^{\alpha}\big[\erm^{\alpha}[X_1 \mid X_2]\big]~.\]
\end{theorem}

Note that the tower property also holds for the expectation operator $\E[\cdot]$, but is violated by most common risk measures, including VaR, CVaR, and EVaR. Despite its many nice features, ERM also has several undesirable properties. It is not positively-homogeneous: $\erm^{\alpha}[c\cdot X]\neq c\cdot \erm^{\alpha}[X]$, for $c\geq 0$, which means that $\erm^{\alpha}[X]$ does not scale linearly with $X$. Moreover, ERM is difficult to interpret and its risk level $\alpha$ is not readily comparable to the risk levels of VaR and CVaR. 

EVaR was proposed to address some of the shortcomings of ERM. EVaR with confidence parameter $\beta\in [0,1)$, for a random variable $X \in \mathbb{X}$, is defined as~\cite{Follmer2011,Ahmadi-Javid2012}
\begin{equation} \label{eq:defn_evar}
\evar^{\beta}[X] \;=\; \sup_{\alpha > 0}\, \left(\erm^\alpha[X] + \alpha^{-1} \cdot  \log(1-\beta)\right).
\end{equation}
The supremum in~\eqref{eq:defn_evar} is achieved for any $X$ with a bounded support. Although EVaR is not dynamically consistent, we show in \cref{sec:rasr-evar} that it can be optimized using a DP by representing it in terms of ERM. Unlike ERM, EVaR is positively-homogeneous, and thus, coherent, which makes its riskiness independent of the scale of the random variable. Moreover, the meaning of its risk level $\beta$ is consistent with those used in VaR and CVaR, with $\evar^0[X] = \E[X]$ and $\lim_{\beta\rightarrow 1}\evar^\beta[X] = \operatorname{ess}\inf[X]$. Finally, since $\evar^{\beta}[X] \leq \cvar^{\beta}[X] \leq \var^{\beta}[X]$, EVaR can be interpreted as the tightest conservative approximation that can be obtained from the Chernoff inequality for VaR and CVaR \cite{Ahmadi-Javid2012}.


\section{RASR-ERM Framework} 
\label{sec:rasr-erm}

In this section, we describe our RASR formulation with the entropic risk measure (ERM), which we refer to as RASR-ERM. In particular, we show that the RASR-ERM objective can be optimized using a novel DP formulation with time-dependent risk. We also establish fundamental properties for the optimal policies of this formulation. The proofs of all the results of this section are in \cref{app:sec:ERM}.

We adopt the soft-robust RL model with dynamic uncertainty. Thus, we assume that the transition model $P = (P_t)_{t=0}^{T-1}$ is a collection of random variables as described in \cref{sec:preliminaries}. Following the RASR objective in~\eqref{eq:RASR}, the RASR-ERM objective is to maximize the ERM of the total return with \emph{both} model uncertainty (epistemic) and random dynamics (aleatory), and is formally defined as 
\begin{equation} \label{eq:RASR-erm-return}
  \max_{\pi\in\Pi_{HR}}  \erm^{\alpha} \left[ \mathfrak{R}_{T}^\pi \right] \;=\; \max_{\pi\in\Pi_{HR}} \erm^{\alpha} \big[\erm^{\alpha} \left[ \mathfrak{R}_{T}^\pi \mid P \right] \big].
\end{equation}
The ERM on the LHS of~\eqref{eq:RASR-erm-return} applies to epistemic and aleatory uncertainties simultaneously and equals to the nested ERM formula on the RHS by \cref{thm:tower-erm}. Compared with~\eqref{eq:risk-averse-rl-ent}, the optimization in~\eqref{eq:RASR-erm-return} involves risk-aversion to the model (epistemic) uncertainty. Compared to~\eqref{eq:soft_robust_return}, the aleatory uncertainty in the return random variable, $\mathfrak{R}_{T}^\pi\mid P$, is modeled by the same risk measure (ERM in place of $\E[\cdot]$) as the one used to model the risk associated with the epistemic (model) uncertainty.  We refer to an optimal solution to~\eqref{eq:RASR-erm-return} as an \emph{optimal policy} $\pi\opt = (\pi_t\opt)_{t=0}^{T-1}$. To simplify the exposition, we restrict our attention in~\eqref{eq:RASR-erm-return} to Markov deterministic policies, because the DP formulation that we derive in \cref{subsec:DP-RASR-ERM} shows that history-dependent or randomized policies offer no advantage in RASR-ERM.


\subsection{Dynamic Program Formulation for RASR-ERM}
\label{subsec:DP-RASR-ERM}

\journal{For the journal: we should define the Bellman operator separately to be able to use it in infinite horizon derivations.}

Before deriving DP equations for the value function in RASR-ERM, we show a simple, but critical, property of ERM. While ERM is known not to be positively homogeneous, the following new result shows that it has a similar property, if we allow for a change in the risk level.   
\begin{theorem}[Positive Quasi-homogeneity] \label{thm:pos-quasi-homogen}
Let $X\in\mathbb X$ be a random variable. Then, for any constant $c \geq 0$ and $\alpha \ge 0$, we have that
\[
\erm^{\alpha} [c\cdot  X] \;=\; c\cdot \erm^{\alpha\cdot c}[X]\;.
\]
\end{theorem}

With the two ERM properties stated in \cref{thm:tower-erm,thm:pos-quasi-homogen}, we are now ready to propose the value function and DP (Bellman) equations for RASR-ERM. The value function for a policy $\pi$ is the collection $v^\pi=(v^\pi_t)_{t=0}^T$, where $v^\pi_t:\states\rightarrow \Real$ is the value at time step $t$ and is defined as 
\begin{equation} \label{eq:erm-rasr-v}
v^\pi_t(s) \;=\; \erm^{\alpha\cdot \gamma^{t}} \left[ \sum_{t'=t}^{T-1} \gamma^{t'-t} \cdot  R^{\pi}_{t'}  \mid S_t = s \right], \quad \forall s\in\mathcal S.
\end{equation}
We define the \emph{optimal value function} $v\opt=(v\opt_t)_{t=0}^T$ as the value function of an optimal policy $\pi\opt$, $v\opt = v^{\pi\opt}$, and let the terminal value function equal to $v_T^{\pi}(s) = 0$. The definition of $v^{\pi}$ with a time-dependent risk aversion parameter is designed specifically to ensure that $v_0^\pi(s_0) = \erm^{\alpha} \left[ \mathfrak{R}_{T}^\pi \right]$ and that $v^{\pi}$ can be computed using dynamic programming equations, as we show below.

\journal{journal: we need to define the optimal value function here because the definition above is ambiguous for $t>0$ without the dynamic programming results. Also define history-dependent policies. The definitions are in the appendix.}

The dependence of risk level on the time step $t$ in the value function definition~\eqref{eq:erm-rasr-v} is quite important in deriving our DP formulation for RASR-ERM below. As time progresses, the risk level $\alpha\cdot \gamma^t$ decreases monotonically, and the value function in~\eqref{eq:erm-rasr-v} becomes less risk-averse. Recall that in the risk-neutral setting, the risk level is $\alpha = 0$ and $\erm^0[X] = \E[X]$. Similarly, when we set $\alpha=0$ in~\eqref{eq:erm-rasr-v}, the value function becomes independent of $t$. Then, if there is no epistemic uncertainty, the function in~\eqref{eq:erm-rasr-v} reduces to the standard MDP value function. 

The next result states the Bellman equations for RASR-ERM value functions.
\begin{theorem}[Bellman Equations]\label{thm:dynamic-program}
For any policy $\pi\in \Pi_{MR}$, its value function $v^\pi=(v_t^{\pi})_{t=0}^{T}$ defined in~\eqref{eq:erm-rasr-v} is the unique solution to the following system of equations:
\begin{equation}\label{eq:v-erm-pi}
v_t^{\pi}(s) \;=\; \erm^{\alpha \cdot \gamma^t} \left[r(s,A) + \gamma\cdot v_{t+1}^{\pi}(S') \right], \quad \forall s\in \states,\; \forall t\in\{0,\ldots,T-1\},
\end{equation}
where $ A \sim \pi_t(\cdot|s)$, $S' \sim \bar{p}_t(\cdot|s,A)$, $\bar{p}_t(s'|s,a) = \E[P_t(s'|s,a)]$, and $v_T^{\pi}(s) = 0$ for each $s\in\states$. 
Moreover, the optimal value function $v\opt=(v\opt_t)_{t=0}^T$ (defined previously) satisfies $v_T\opt(s) = 0$ and is the unique solution to 
\begin{equation}\label{eq:v-erm-opt}
 v\opt_t(s) \; =\;  \max_{a\in \actions}\,  \erm^{\alpha \cdot \gamma^t} \left[r(s,a) + \gamma \cdot  v_{t+1}\opt(S') \right], \quad \forall s\in\states,\;S'\sim \bar{p}_t(\cdot|s,a).
\end{equation}
\end{theorem}
Note that the ERM operator in~\eqref{eq:v-erm-pi} and~\eqref{eq:v-erm-opt} applies to the random variables $A$ and $S'$. 

\Cref{thm:dynamic-program} suggests several new important and surprising properties for the RASR-ERM objective~\eqref{eq:RASR-erm-return}. The first property that follows from the DP equations in \cref{thm:dynamic-program} is that the RASR-ERM objective~\eqref{eq:RASR-erm-return} is equivalent to a risk-averse RL problem with the mean posterior transition model $\bar{p}$ defined in \cref{thm:dynamic-program}. 
\begin{corollary} \label{thm:average-model}
The return for each policy $\pi \in \Pi_{MR}$ satisfies that
\[
\erm^{\alpha} \big[\erm^{\alpha}[\mathfrak{R}_{T}^\pi \mid P]\big] = \erm^{\alpha}\left[\mathfrak{R}_{T}^\pi \mid P = \bar{p}\right],
\]
where $\bar{p} = (\bar{p}_t)_{t=0}^{T-1}$ is defined as in \cref{thm:dynamic-program}.
\end{corollary}

The second important result that follows from \cref{thm:dynamic-program} is that there exists an optimal Markov (as opposed to history-dependent) deterministic policy for the RASR-ERM objective~\eqref{eq:RASR-erm-return}, which is greedy w.r.t.~the optimal value function $v^\star$ defined by~\eqref{eq:v-erm-opt}. However, unlike in risk-neutral MDPs~\cite{Puterman2005}, the optimal RASR-ERM policy may be time-dependent even when the horizon $T$ is large or inifnite.
\begin{theorem} \label{th:optimal_deterministic}
There exists a Markov deterministic optimal policy $\pi\opt = (\pi_t\opt)_{t=0}^{T-1}\in\Pi_{MD}$ for the optimization problem~\eqref{eq:RASR-erm-return}, which is greedy w.r.t.~the optimal value function $v\opt$ defined by~\eqref{eq:v-erm-opt}: 
\begin{equation}\label{eq:pol-greedy}
\pi\opt_t(s) \in \argmax_{a\in\actions} \,\erm^{\alpha \cdot \gamma^t} \big[ r(s,a) + \gamma  \cdot v\opt_{t+1}(S')\big], \quad \forall s\in \states,\;S' \sim \bar{p}_t(\cdot|s,a).
\end{equation}
\end{theorem}

The existence of optimal deterministic policies in RASR-ERM is surprising since many risk-averse and soft-robust formulations require randomization~\cite{Delage2019,Lobo2021,Steimle2021a}. \journal{Maybe elaborate here in the journal version.} Also surprisingly, RASR-ERM does not admit a stationary optimal policy in the infinite-horizon discounted setting because of the time-dependent risk-level in RASR-ERM dynamic program. Finally, note that the above results provide stronger guarantees than the DP equations for the existing soft-robust MDP formulations~\cite{Eriksson2020,Lobo2021}. Using time-dependent risk-levels in our DP formulation in \cref{thm:dynamic-program} guarantees that the optimal value function solves the objective~\eqref{eq:RASR-erm-return} optimally. This is in contrast to other soft-robust formulations, which do not admit dynamic program formulations~\cite{Lobo2021}. 


\subsection{Algorithms for Optimizing RASR-ERM}
\label{subsec:VI-RASR-ERM}

We now turn to algorithms that can compute RASR-ERM value functions and policies. With the \emph{finite-horizon} objective ($T<\infty$), the optimal value function can be computed by adapting the standard value iteration (VI) to this setting. This algorithm computes the optimal value function $v\opt_t$ backwards in time $t = T, T-1, \dots , 0$ according to~\eqref{eq:v-erm-opt}. The optimal policy is greedy to $v\opt$ and can be computed by solving the discrete optimization problem in~\eqref{eq:pol-greedy}. We include the full algorithms in the appendix in \cref{app:sec:ERM}.

Solving the \emph{infinite-horizon} problem is considerably more challenging than the finite-horizon problem, because the risk level $\alpha$ and the optimal policy are in general time dependent. The simplest way to address this issue is to simply truncate the horizon to some $T' < \infty$ and resort to an arbitrary policy for any $t > T'$. The significant limitation to truncating the horizon is that $T'$ may need to be very large to achieve a reasonably-small approximation error. As is standard in infinite-horizon settings, we assume that the transition probabilities are \emph{stationary}. That is, there exists a transition function $P$ such that $P = P_t$ for all $t = 0, \dots, \infty$. As a result, the mean transition probability $\bar{p}_t$ is also stationary and we omit the subscript $t$ throughout this section. 

In \cref{alg:rasr-vi-inf}, we propose an approximation that is superior to a truncated planning horizon. The algorithms works as follows. First, it computes the optimal stationary risk-neutral value function $v^{\infty}$ and policy $\pi^{\infty}$ using value iteration or policy iteration~\cite{Puterman2005}. The policy $\pi^\infty$ is used for all time steps $ t > T'$ and the value function $v^{\infty}$ is used to approximate $v\opt_{T'}$. This approach takes an advantage of the fact that the risk level $\alpha \cdot \gamma^t$ in~\eqref{eq:v-erm-opt} approaches $0$ as $t \to  \infty$. This means that the ERM value function becomes ever closer to the optimal risk-neutral discounted value function $v^{\infty}$.  

\begin{algorithm}
    \KwIn{Planning horizon $T' < \infty$, risk level $\alpha > 0$}
    \KwOut{Optimal policy $\;\hat\pi\opt = (\hat\pi_t\opt)_{t=0}^\infty\;$ and value function $\;\hat{v}\opt = (\hat{v}_t\opt)_{t=0}^\infty$}
    Compute optimal $\;v^{\infty}\;$ and $\;\pi^{\infty}\;$ as a solution to the infinite-horizon discounted MDP with $\;\bar{p}$ \;
    Compute $\;(\tilde{v}\opt_t)_{t=0}^{T'}\;$ and $\;(\tilde{\pi}\opt_t)_{t=0}^{T'-1}\;$ using ~\eqref{eq:v-erm-opt} and~\eqref{eq:pol-greedy} with horizon $\;T'\;$ and terminal value $\;\tilde{v}\opt_{T'} = v^{\infty}$\;
    Construct a policy $\;(\hat{\pi}\opt_{t})_{t=0}^{\infty}\;$, where $\;\hat{\pi}\opt_t = \pi^{\infty}$ when $\;t \ge T'\;$  and $\;\hat{\pi}\opt_t = \tilde{\pi}\opt_t\;$, otherwise \;
    Construct $\;\hat{v}\opt\;$ analogously to $\;\hat{\pi}\opt$\;
\Return{$\;\hat\pi\opt\;$, $\;\hat{v}\opt$}
  \caption{VI for infinite-horizon RASR-ERM} \label{alg:rasr-vi-inf}
\end{algorithm}

To quantify the quality of the policy $\hat\pi\opt$ returned by \cref{alg:rasr-vi-inf}, we now derive a bound on its performance loss. In particular, we focus on how quickly the error decreases as a function of the planning horizon $T'$. This bound can be used both to determine the planning horizon and to quantify the improvement of \cref{alg:rasr-vi-inf} over simply truncating the planning horizon. 
\begin{theorem}\label{thm:approx-error}
The performance loss of a policy $\hat{\pi}\opt$ returned by \cref{alg:rasr-vi-inf} for a discount factor $\gamma < 1$ decreases with $T'$ as
\begin{align*}
\erm^{\alpha} \big[ \mathfrak{R}_{\infty}^{\pi\opt} \mid P=\bar{p} \big] - 
\erm^{\alpha} \big[ \mathfrak{R}_{\infty}^{\hat{\pi}\opt} \mid P=\bar{p} \big]
  \;\le\;
c \cdot \gamma^{2 T'} ~,
\end{align*}
where $\pi\opt$ is optimal in~\eqref{eq:RASR-erm-return} and $c = 8^{-1} \alpha \cdot (\triangle r)^{2}  (1-\gamma)^{-2}$.
\end{theorem}
The proof of \cref{thm:approx-error} uses the Hoeffding's lemma to bound the error between ERM and the expectation and propagates the error backwards using standard dynamic programming techniques. 

Analysis analogous to \cref{thm:approx-error} shows that when one truncates the horizon at $T'$ and follows an arbitrary policy thereafter, the performance loss decreases proportionally to $\gamma^{T'}$ as opposed to $\gamma^{2 T'}$. As a result, truncating a policy requires at least double the planning horizon $T'$ to achieve the same approximation guarantee as \cref{alg:rasr-vi-inf}.

In practice, one can compute bounds that are tighter than \cref{thm:approx-error} by computing both an upper bound on the optimal value function and a lower bound on the value of the policy. It is easy to see that $v^{\infty }$ is an upper bound on $v\opt$, which can be used to compute an upper bound on $v_0\opt$ and, therefore, an upper bound on the performance loss. We give more details in \cref{app:sec:ERM}.


\section{RASR-EVaR Framework} \label{sec:rasr-evar}

In this section, we introduce and analyze RASR with the EVaR objective, which we refer to as the RASR-EVaR framework. As mentioned in \cref{sec:preliminaries}, EVaR is preferable to ERM because it is coherent, positively-homogenous, interpretable, and comparable with VaR and CVaR. The main challenge with RASR-EVaR is that EVaR does not satisfy the tower property in \cref{thm:tower-erm} (or, equivalently, it is \emph{not} dynamically consistent), and thus, cannot be directly optimized using a DP.  Our main contribution here is to show that despite this issue, it is possible to solve RASR-EVaR by extending the algorithms developed for RASR-ERM in \cref{sec:rasr-erm}. The detailed proofs of all the results of this section are reported in \cref{app:sec:EVaR}.

The RASR-EVaR formulation assumes the same setting as in~\eqref{eq:RASR-erm-return} with the following objective:
\begin{equation} \label{eq:RASR-evar-return}
\max_{\pi\in\Pi_{HR}}  \evar^{\beta} \left[ \mathfrak{R}_{T}^\pi \right] ~.
\end{equation}
The EVaR operator in~\eqref{eq:RASR-evar-return} applies simultaneously to both epistemic and aleatory uncertainties over returns. Because EVaR does not satisfy the tower property, it is impossible to rewrite~\eqref{eq:RASR-evar-return} using separate risk for the aleatory and epistemic uncertainty, similarly to~\eqref{eq:RASR-erm-return}. We use $\pi\opt$ throughout this section to denote an optimal policy in~\eqref{eq:RASR-evar-return}. Prior work on EVaR in MDPs focuses exclusively on the nested (or Markov) risk formulation~\cite{Ahmadi2021a,Ahmadi2021b,Dixit2021}, which is typically overly conservative~\cite{Iancu2015a}.

Our main idea is to reformulate the RASR-EVaR objective in~\eqref{eq:RASR-evar-return} using the EVaR definition in~\eqref{eq:defn_evar} in terms of a sequence of ERM formulations as
\begin{equation}\label{eq:rasr-evar-ref}
\max_{\pi\in\Pi_{HR}}  \evar^{\beta} \left[ \mathfrak{R}_{T}^\pi \right] = \max_{\alpha \ge 0} \underbrace{\max_{\pi\in\Pi_{HR}}\, \left(\erm^\alpha[\mathfrak{R}_{T}^\pi] + \alpha^{-1} \cdot \log(1-\beta) \right)}_{= h(\alpha)}~.
\end{equation}
The equality above follows by swapping the order of maximization operators and since $\mathfrak{R}_{T}^\pi$ is bounded, the $\sup$ is attained. The equality in~\eqref{eq:rasr-evar-ref} indicates that any RASR-EVaR optimal policy must also be RASR-ERM optimal for some $\alpha\opt$. This allows us to directly carry over the following results from the RASR-ERM setting to RASR-EVaR.
\begin{theorem} \label{thm:equivalence-evar-erm}
Let $\pi\opt$ be an optimal solution to RASR-EVaR in~\eqref{eq:RASR-evar-return}. Then, there exists a risk level $\alpha\opt$ such that $\pi\opt$ is optimal in RASR-ERM (Eq.~\ref{eq:RASR-erm-return}) with $\alpha = \alpha\opt$.
\end{theorem}
\cref{thm:equivalence-evar-erm} combined with the properties of RASR-ERM, shown in \cref{sec:rasr-erm}, can be used to establish the following properties for RASR-EVaR.
\begin{corollary} \label{cor:evar-markov}
The RASR-EVaR setting (Eq.~\ref{eq:RASR-evar-return}) has a Markov and deterministic optimal policy $\pi\opt\in\Pi_{MD}$. Moreover, for any policy $\pi \in \Pi_{MR}$, the RASR-EVaR objective~\eqref{eq:RASR-evar-return} equals to 
\[
\evar^{\beta} \left[  \mathfrak{R}_{T}^\pi \right] =  \evar^{\beta} \left[ \mathfrak{R}_{T}^\pi \mid P = \bar{p} \right]\;,
\]
where $\bar{p}$ is defined as in \cref{thm:average-model}.
\end{corollary}

We are now ready to describe our algorithms for solving the RASR-EVaR objective given in \cref{alg:ERM_EVAR}. The algorithm takes advantage of the fact that the optimization problem $\max_{\alpha \ge 0} h(\alpha)$ is single-dimensional. The algorithm searches a grid  of candidate $\alpha$ values. Each $h(\alpha)$ is computed via the RASR-ERM algorithms described in \cref{sec:rasr-erm}.

\begin{algorithm} 
    \KwIn{Discretized risk levels $\Lambda = \{\;\alpha_0 \ge \dots \ge \alpha_K > 0 \}$}
    \KwOut{RASR-EVaR optimized policy $\;\hat{\pi}\opt$}
    \For{$k = 0,\ldots,K$}{
    Compute $\;v^k\;$ and $\;\pi^k\;$ by solving the RASR-ERM problem with risk level $\;\alpha_k$
    }
    Let $\;k\opt  \gets \argmax_{k=0,\ldots,K} \;v_0^k(s_0) + \alpha_k^{-1} \cdot \log (1-\beta)$\;
    \Return{Policy $\;\hat{\pi}\opt = \pi^{k\opt}$}
    \caption{Algorithm for RASR-EVaR}    \label{alg:ERM_EVAR}
\end{algorithm}

\cref{alg:ERM_EVAR} resorts to discretizing $\alpha$ values because $h(\alpha)$ is non-concave in general (see \cref{prop:non-concave}), and thus, cannot be maximized using more efficient algorithms. Our key contribution is that we use the properties of $h$ to show that a specific discrete grid of points can be used to compute a good solutions without an excessive computational burden. 
\begin{theorem} \label{thm:rasr-evar-bound}
Suppose that \cref{alg:ERM_EVAR} uses discretized risk levels $\Lambda$ that satisfy that
\[
  \alpha_k =  \frac{-\log(1-\beta)}{k \cdot \delta },\quad  k\in \{0,\ldots,K\}, \qquad\quad  K \geq \sqrt{\frac{-\log(1-\beta)}{8 }} \cdot \frac{\triangle r}{ (1-\gamma) \cdot  \delta }~,
\]
where $\delta > 0$ and $\alpha_0 = \infty$ (note $\erm^\infty[\cdot] = \operatorname{ess}\inf[\cdot]$). Then, the performance loss of the policy $\hat{\pi}\opt$ returned by \cref{alg:ERM_EVAR} is bounded by $\;\evar^{\beta} \left[ \mathfrak{R}_{\infty}^{\pi\opt} \mid P=\bar{p} \right] - 
\evar^{\beta} \left[ \mathfrak{R}_{\infty}^{\hat{\pi}\opt} \mid P=\bar{p} \right]
\le \delta$.
\end{theorem}

One can accelerate \cref{alg:ERM_EVAR} by realizing that \cref{alg:rasr-vi-inf} computes value functions for multiple risk levels $\alpha, \gamma\alpha, \gamma^2\alpha, \ldots $. For instance, running  \cref{alg:rasr-vi-inf} with $\alpha = 0.5$ computes $v_0$ with a risk $\alpha=0.5$, $v_1$ with a risk $\alpha=0.5\gamma$, $v_2$ with a risk level $\alpha=0.5\gamma ^2$ and so on. This observation can significantly reduce the computational effort while introducing an additional small error due to the effective approximate horizon $T'$ being different for different risk levels $\alpha$. Given that this is the first work proposing and optimizing RASR-EVaR, we focus on the conceptually simple \cref{alg:ERM_EVAR} and leave computational improvements for future work.

\section{Empirical Evaluation} \label{sec:empirical}

In this section, we evaluate our RASR framework empirically on several MDPs used previously to evaluate soft-robust and risk-averse algorithms. The empirical evaluation focuses on RASR-EVaR for two reasons. First, as discussed in \cref{sec:preliminaries}, EVaR is a more practical risk measure than ERM because it is closely related to the popular VaR and CVaR. Second, any RASR-EVaR optimal policy is also a RASR-ERM policy for some $\alpha$ optimal in~\eqref{eq:rasr-evar-ref}. We provide additional results, information, and details in \cref{sec:experiments-detail}.

We now describe the experimental setup. As the primary metric for the comparison, we use $\evar^{0.99}[\mathfrak{R}_{\infty}^{\pi}]$ for a policy $\pi$ computed by RASR-EVaR and other baseline algorithms. For the sake of completeness, we also compare the risk computed using VaR and CVaR, two common risk measures. The epistemic uncertainty in our experiments follows the dynamic model described in \cref{sec:preliminaries}. We use the following three domains from the robust RL literature to evaluate the algorithms: \emph{river-swim}~\cite{Behzadian2021}, \emph{population}~\cite{Russel2019beyond}, and \emph{inventory}~\cite{Behzadian2021}. The \emph{river-swim} problem is used to test whether the algorithms are sufficiently risk-averse. It involves small epistemic uncertainty with a significant impact on the return. In contrast, we use the \emph{population} problem to test if the algorithms are overly risk-averse. The epistemic uncertainty is large but makes a small difference in the overall return. Finally, the \emph{inventory} domain combines the characteristics of both these domains.

To understand how well RASR-EVaR performs, we compare the policy it computes with several related methods. Even though these baselines were designed to be risk-averse to the epistemic uncertainty, comparing RASR-EVaR with them helps us understand the importance of jointly optimizing for epistemic and aleatory uncertainties. The \emph{Naive} algorithm computes the ERM value function by solving a dynamic program akin to \cref{thm:dynamic-program}, but with risk $\alpha$ that is constant across time. Algorithms \emph{Erik}~\cite{Eriksson2020}, \emph{Derman}~\cite{Derman2018}, \emph{BCR}~\cite{Behzadian2021}, \emph{RSVF}~\cite{Russel2019beyond}, and \emph{SRVI}~\cite{Lobo2021} originated in the robust RL literature and their objectives are summarized in \cref{sec:related_work} and \cref{sec:addit-relat-work}. \emph{BCR} and \emph{RSVF} are two recent algorithms that have been proposed to optimize the percentile objective (which is equivalent to VaR). \emph{SRVI} optimizes a CVaR objective. Finally, we also compare with a risk-averse MDP algorithm by Chow et al.~\cite{Chow2015} (\emph{Chow}), which is related to RASR-ERM. It augments the state space in a way that is superficially similar to our time-dependent value functions. We use risk-averse methods with the average model $\bar P$ as described in \cref{thm:average-model,cor:evar-markov}. The downsides of \emph{Chow} are that the augmented state space they use is infinite and their policies are history dependent.

\journal{Emphasize that this is also the first paper that does EVaR in reinforcement learning in general.}

\begin{table}
\centering
\begin{tabular}{l|rrr}
  \toprule
 \multicolumn{1}{c|}{Method} & \multicolumn{1}{c}{RS} & \multicolumn{1}{c}{POP} & \multicolumn{1}{c}{INV} \\
  \midrule
\textbf{RASR} & \textbf{50} & \textbf{-7020} & \textbf{294} \\
Naive & \textbf{50} & -8291 & 290 \\
Erik & 45 & -8628 & 290 \\
Derman & 7 & -7259 & 287 \\
RSVF & 45 & -8874 & 257 \\
BCR & 34 & -8731 & 281 \\
SRVI & 34 & -8714 & 280 \\
Chow & 23 & -7238 & 290 \\
  \bottomrule
\end{tabular}
\caption{$\evar^{0.99}[\mathfrak{R}^{\pi}_{\infty}]$ of the policy $\pi$ returned by each method.}
\label{tab:evar_099}
\end{table}

\begin{table}
\centering
\begin{tabular}{l|l|l|l}
\toprule
& & \multicolumn{2}{|c}{Risk Measure}  \\    
Method  & Object. &  Epistemic & Aleatory \\
\midrule
\textbf{RASR} & Disc. & EVaR & EVaR \\
Erik~\cite{Eriksson2020} & Disc. & ERM & E  \\
Derman~\cite{Derman2018} & Aver. & E & E  \\
RSVF~\cite{Russel2019beyond} & Disc. & VaR & E  \\
BCR~\cite{Behzadian2021} & Disc. & VaR & E  \\
SRVI~\cite{Lobo2021} & Disc. & CVaR & E \\
Chow~\cite{Chow2015} & Disc. &  -- & CVaR \\
\bottomrule
\end{tabular}
  \caption{Summary of the soft-robust and risk-averse models in the MDP/RL literature.}
  \label{tab:related}
\end{table}

\begin{figure}
\centering
\includegraphics[width=0.48\linewidth]{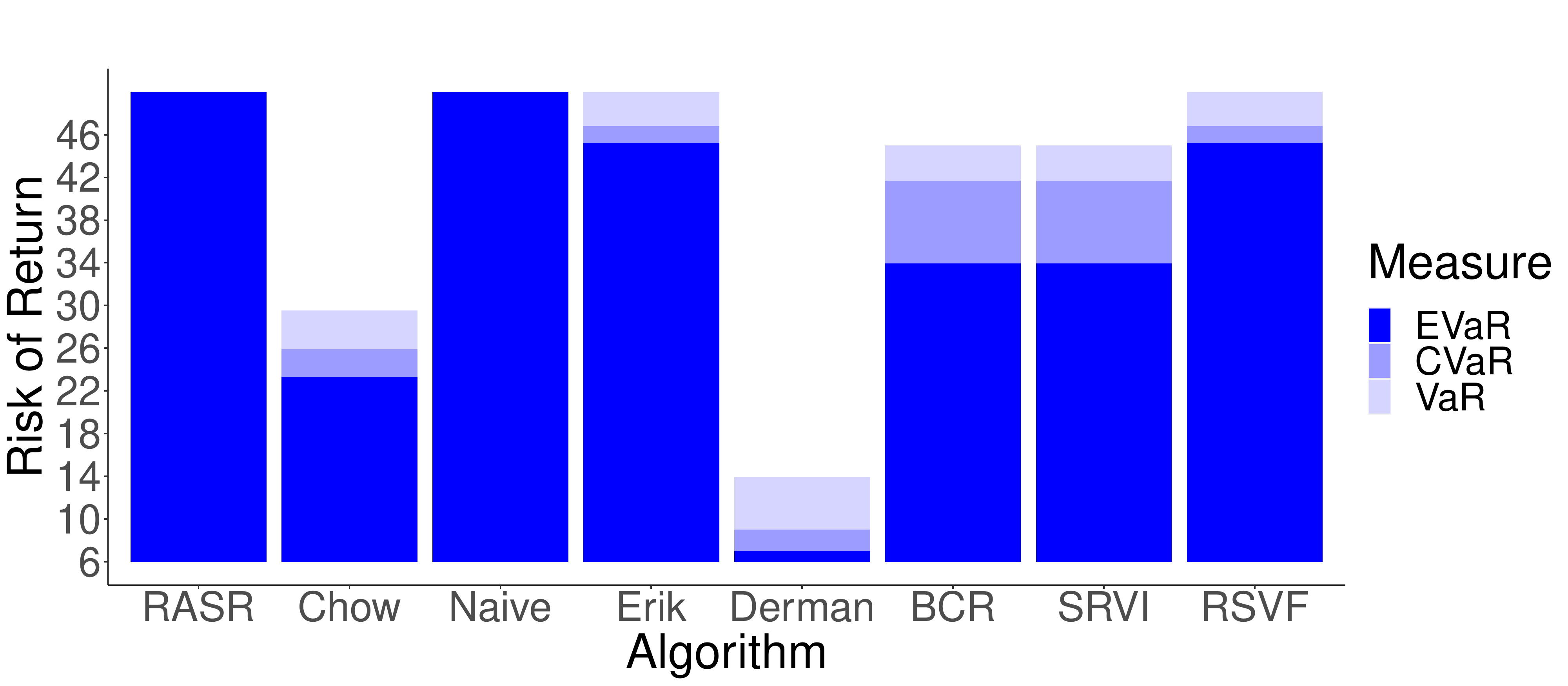}
\includegraphics[width=0.48\linewidth]{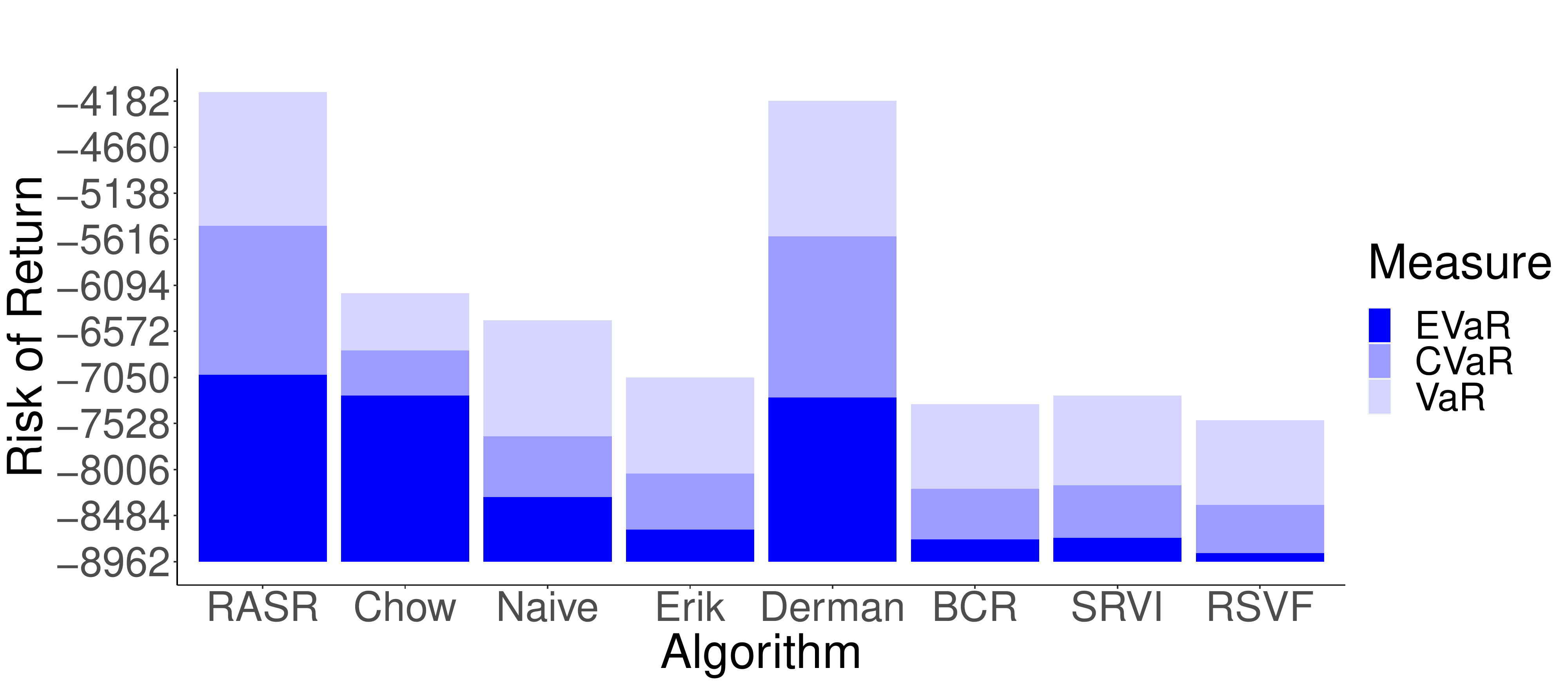}
\caption{$\psi^{0.99}[\mathfrak{R}^{\pi}_{\infty}]$ in \emph{river-swim} (left) and \emph{population} (right) problems~. }
\label{fig:combined_barplot}
\end{figure}

\cref{tab:evar_099} summarizes the risk $\evar^{0.99}[\mathfrak{R}^{\pi}_{\infty}]$ for policies $\pi$ computed by RASR-EVaR and the baseline algorithms described above; please see \cref{fig:combined_barplot} and \cref{sec:experiments-detail} for an evaluation with other risk measures. The results show that RASR-EVaR chooses the \emph{appropriate} level of risk-aversion across all domains. The plots in \cref{fig:combined_barplot} help to visualize the situation for two of the domains. \emph{Derman} is risk neutral and performs particularly poorly in \emph{river-swim}, which has small but impactful epistemic risk. Risk averse algorithms, like \emph{RSVF} and \emph{Erik}, perform well in this domain. In contrast, \emph{Derman} performs well in \emph{population}, which involves large but inconsequential epistemic uncertainty. The risk-averse algorithms (\emph{RSVF}, \emph{Erik}, \emph{BCR}) put too much emphasis on the epistemic uncertainty in this domain and compute policies that are too conservative. 

Examining the results in \cref{fig:combined_barplot} closer leads one to several other important conclusions. First, the figures show that RASR-EVaR outperforms other algorithms even when the risk is evaluated using CVaR or VaR and may be a viable approximate approach optimizing these other risk measures. Second, the results in \cref{fig:combined_barplot} point to the importance of using the time dependent risk in the dynamic program equations. The \emph{Naive} algorithm performs poorly in the \emph{population} domain.

\section{Related Work} 
\label{sec:related_work}


Our RASR framework falls under the broader umbrella of robust and soft-robust MDP and RL. Robust optimization is a methodology that reduces the sensitivity of the solution to model errors~\cite{Ben-Tal2009} and has been extensively studied in MDP~\cite{Nilim2005,Iyengar2005,Wiesemann2013,Ho2021a} and RL~\cite{Xu2012,Petrik2014,Russel2019beyond,Grand-Clement2021b}. Since robust MDPs often compute policies that are overly conservative, soft-robust (also known as Bayesian robust, light robust, or multi-model objectives) formulations were proposed as an alternative that can better balance robustness and the quality of an average solution (e.g.,~\cite{Ben-Tal2010,Derman2018,Mankowitz2019,Buchholz2020}). Soft-robust algorithms replace the worst-case objective of robust optimization with risk-aversion to some distribution over uncertain models. \Cref{tab:related} summarizes representative soft-robust and risk-averse algorithms studied in the MDP/RL literature. We defer a more comprehensive overview of related work to \cref{sec:addit-relat-work}.

Risk-averse MDP methods account only for the aleatory uncertainty in the return and do not explicitly consider the error in the model. The risk-averse formulations that are most closely related to our work use ERM. This list includes the results in the average reward~\cite{Borkar2014,Borkar2002a,Borkar2002Qrisk} and those in the undiscounted finite-horizon settings~\cite{Fei2020,Nass2020,Dowson2021}. Note that some of these papers address risk-aversion in stochastic programming and not in MDPs~\cite{Dowson2021}. To the best of our knowledge, none of the prior work has studied ERM in the discounted case. We believe this is because ERM is not positive-homogeneous, which complicates using it with a discount factor, as shown in \cref{thm:pos-quasi-homogen}. Moreover, we are unaware of any EVaR adaptation of these earlier ERM algorithms. Most other formulations for risk-averse RL are based on VaR and CVaR~\cite{Borkar2014,Chow2014,Tamar2015,Chow2018}, which are not dynamically consistent and are NP hard to optimize. To build a DP in these formulations, one must augment the state space and optimize over a continuously infinite variable~\cite{Bauerle2011,Chow2014,Chow2015,Pflug2016}, which significantly complicates the computation in comparison with the time-dependent value functions in RASR-ERM. Finally, existing application of EVaR to MDPs have been limited to the nested (or Markov) formulation, which embeds the risk measure directly into the Bellman operator~\cite{Ahmadi2021a,Ahmadi2021b,Dixit2021}. The nested EVaR risk measure differs from the ordinary EVaR and generally approximates if only very poorly~\cite{Iancu2015a}.

\section{Conclusion and Future Work} 
\label{sec:conclusion} 

We proposed RASR, a framework that can mitigate the risk associated with both model uncertainty (epistemic) and random dynamics (aleatory) in MDPs. We studied RASR with two separate risk measures: ERM and EVaR. In RASR-ERM, we derived the first exact DP formulation for ERM in discounted MDPs. We also showed that there optimal value function and deterministic Markov policies exist and can be computed using value iteration. For RASR-EVaR, we show that the RASR-EVaR objective can be optimized by reducing it to multiple RASR-ERM problems. Our empirical results highlight the utility of our RASR algorithms. 

Future directions include scaling our RASR algorithms beyond tabular MDPs and dynamic epistemic uncertainty. It is also essential to better understand the relation between RASR and regularized (robust) MDPs~\cite{Derman2021,neu2017unified,Geist2019}. 




\bibliography{biblio}

\begin{thebibliography}{65}
\providecommand{\natexlab}[1]{#1}
\providecommand{\url}[1]{\texttt{#1}}
\expandafter\ifx\csname urlstyle\endcsname\relax
  \providecommand{\doi}[1]{doi: #1}\else
  \providecommand{\doi}{doi: \begingroup \urlstyle{rm}\Url}\fi

\bibitem[Ahmadi et~al.(2021{\natexlab{a}})Ahmadi, Rosolia, Ingham, Murray, and
  Ames]{Ahmadi2021a}
M.~Ahmadi, U.~Rosolia, M.~D. Ingham, R.~M. Murray, and A.~D. Ames.
\newblock Constrained risk-averse {{Markov}} decision processes.
\newblock \emph{Proceedings of the AAAI Conference on Artificial Intelligence},
  35\penalty0 (13):\penalty0 11718--11725, 2021{\natexlab{a}}.

\bibitem[Ahmadi et~al.(2021{\natexlab{b}})Ahmadi, Rosolia, Ingham, Murray, and
  Ames]{Ahmadi2021b}
M.~Ahmadi, U.~Rosolia, M.~D. Ingham, R.~M. Murray, and A.~D. Ames.
\newblock Risk-averse decision making under uncertainty, 2021{\natexlab{b}}.

\bibitem[Ahmadi-Javid(2012)]{Ahmadi-Javid2012}
A.~Ahmadi-Javid.
\newblock Entropic value-at-risk: A new coherent risk measure.
\newblock \emph{Journal of Optimization Theory and Applications}, 2012.

\bibitem[Angelotti et~al.(2021)Angelotti, Drougard, and Chanel]{Angelotti2021}
G.~Angelotti, N.~Drougard, and C.~P.~C. Chanel.
\newblock Exploitation vs caution: Risk-sensitive policies for offline
  learning.
\newblock \emph{arXiv:2105.13431 [cs, eess]}, 2021.

\bibitem[Artzner et~al.(1999)Artzner, Delbaen, Eber, and Heath]{Artzner1999}
P.~Artzner, F.~Delbaen, J.-m. Eber, and D.~Heath.
\newblock Coherent measures of risk.
\newblock \emph{Mathematical Finance}, 9:\penalty0 203--228, 1999.

\bibitem[Artzner et~al.(2004)Artzner, Delbaen, Eber, Heath, and
  Ku]{Artzner2004}
P.~Artzner, F.~Delbaen, J.~M. Eber, D.~Heath, and H.~Ku.
\newblock Coherent multiperiod risk adjusted values and {{Bellman's}}
  principle.
\newblock \emph{Annals of Operations Research}, 2004.

\bibitem[Bauerle and Ott(2011)]{Bauerle2011}
N.~Bauerle and J.~Ott.
\newblock Markov {{Decision Processes}} with {{Average-Value-at-Risk}}
  criteria.
\newblock \emph{Mathematical Methods of Operations Research}, 74\penalty0
  (3):\penalty0 361--379, 2011.

\bibitem[Behzadian et~al.(2021)Behzadian, Russel, Ho, and
  Petrik]{Behzadian2021}
B.~Behzadian, R.~Russel, C.~P. Ho, and M.~Petrik.
\newblock Optimizing percentile criterion using robust {{MDPs}}.
\newblock In \emph{International {{Conference}} on {{Artificial Intelligence}}
  and {{Statistics}} ({{AIStats}})}, 2021.

\bibitem[Ben-Tal and Teboulle(2007)]{Ben-Tal2005}
A.~Ben-Tal and M.~Teboulle.
\newblock {An Old-New Concept of Convex Risk Measures: The Optimized Certainty
  Equivalent}.
\newblock \emph{Mathematical Finance}, 17:\penalty0 449--476, 2007.

\bibitem[Ben-Tal et~al.(2009)Ben-Tal, Ghaoui, and Nemirovski]{Ben-Tal2009}
A.~Ben-Tal, L.~E. Ghaoui, and A.~Nemirovski.
\newblock \emph{{Robust Optimization}}.
\newblock Princeton University Press, 2009.

\bibitem[Ben-Tal et~al.(2010)Ben-Tal, Bertsimas, and Brown]{Ben-Tal2010}
A.~Ben-Tal, D.~Bertsimas, and D.~B. Brown.
\newblock A soft robust model for optimization under ambiguity.
\newblock \emph{Operations Research}, 2010.

\bibitem[Borkar and Jain(2014)]{Borkar2014}
V.~Borkar and R.~Jain.
\newblock Risk-constrained {Markov} decision processes.
\newblock \emph{IEEE Transactions on Automatic Control}, 2014.

\bibitem[Borkar(2002)]{Borkar2002Qrisk}
V.~S. Borkar.
\newblock Q-learning for risk-sensitive control.
\newblock \emph{Mathematics of Operations Research}, 27\penalty0 (2):\penalty0
  294--311, 2002.

\bibitem[Borkar and Meyn(2002)]{Borkar2002a}
V.~S. Borkar and S.~P. Meyn.
\newblock Risk-sensitive optimal control for {{Markov}} decision processes with
  monotone cost.
\newblock \emph{Mathematics of Operations Research}, 27\penalty0 (1):\penalty0
  192--209, 2002.

\bibitem[Boucheron et~al.(2013)Boucheron, Lugosi, and Massart]{Boucheron2013}
S.~Boucheron, G.~Lugosi, and P.~Massart.
\newblock \emph{Concentration {{Inequalities}}: {{A}} Nonasymptotic Theory of
  Independence}.
\newblock {Oxford University Press}, 2013.

\bibitem[Brown et~al.(2020)Brown, Niekum, and Petrik]{Brown2020}
D.~S. Brown, S.~Niekum, and M.~Petrik.
\newblock Bayesian robust optimization for imitation learning.
\newblock In \emph{Advances in {{Neural Information Processing Systems}}
  ({{NeurIPS}})}, 2020.

\bibitem[Buchholz and Scheftelowitsch(2019)]{Buchholz2019}
P.~Buchholz and D.~Scheftelowitsch.
\newblock Light robustness in the optimization of {{Markov}} decision processes
  with uncertain parameters.
\newblock \emph{Computers and Operations Research}, 108:\penalty0 69--81, 2019.

\bibitem[Buchholz and Scheftelowitsch(2020)]{Buchholz2020}
P.~Buchholz and D.~Scheftelowitsch.
\newblock Concurrent {{MDPs}} with finite {{Markovian}} policies.
\newblock In \emph{Measurement, {{Modeling}}, and {{Evaluation}} of
  {{Computing}}}, pages 37--53, 2020.

\bibitem[Chen and Bowling(2012)]{Chen2012}
K.~Chen and M.~Bowling.
\newblock Tractable objectives for robust policy optimization.
\newblock \emph{Advances in Neural Information Processing Systems}, 3:\penalty0
  2069--2077, 2012.

\bibitem[Chow and Ghavamzadeh(2014)]{Chow2014}
Y.~Chow and M.~Ghavamzadeh.
\newblock {Algorithms for CVaR optimization in MDPs}.
\newblock \emph{Advances in Neural Information Processing Systems}, 2014.

\bibitem[Chow et~al.(2015)Chow, Tamar, Mannor, and Pavone]{Chow2015}
Y.~Chow, A.~Tamar, S.~Mannor, and M.~Pavone.
\newblock {Risk-sensitive and robust decision-making : A CVaR optimization
  approach}.
\newblock In \emph{Neural Information Processing Systems (NIPS)}, 2015.

\bibitem[Chow et~al.(2018)Chow, Ghavamzadeh, Janson, and Pavone]{Chow2018}
Y.~Chow, M.~Ghavamzadeh, L.~Janson, and M.~Pavone.
\newblock {Risk-constrained reinforcement learning with percentile risk
  criteria}.
\newblock \emph{Journal of Machine Learning Research}, 2018.

\bibitem[Cvitani{\'{c}} and Karatzas(1999)]{Cvitanic1999}
J.~Cvitani{\'{c}} and I.~Karatzas.
\newblock {On dynamic measures of risk}.
\newblock \emph{Finance and Stochastics}, 1999.

\bibitem[Delage and Mannor(2009)]{Delage2009}
E.~Delage and S.~Mannor.
\newblock Percentile optimization for {Markov} decision processes with
  parameter uncertainty.
\newblock \emph{Operations Research}, 2009.

\bibitem[Delage et~al.(2019)Delage, Kuhn, and Wiesemann]{Delage2019}
E.~Delage, D.~Kuhn, and W.~Wiesemann.
\newblock ``{{Dice}}''-sion-making under uncertainty: When can a random
  decision reduce risk?
\newblock \emph{Management Science}, 65\penalty0 (7):\penalty0 3282--3301,
  2019.

\bibitem[Delbaen(2006)]{Delbaen2006}
F.~Delbaen.
\newblock {The structure of m--stable sets and in particular of the set of the
  risk neutral measures}.
\newblock \emph{In Memoriam Paul-Andr{\'{e}} Meyer}, 2006.

\bibitem[Derman et~al.(2018)Derman, Mankowitz, Mann, and Mannor]{Derman2018}
E.~Derman, D.~J. Mankowitz, T.~A. Mann, and S.~Mannor.
\newblock {Soft-robust actor-critic policy-gradient}.
\newblock \emph{Conference on Uncertainty in Artificial Intelligence}, 2018.

\bibitem[Derman et~al.(2021)Derman, Geist, and Mannor]{Derman2021}
E.~Derman, M.~Geist, and S.~Mannor.
\newblock Twice regularized {{MDPs}} and the equivalence between robustness and
  regularization.
\newblock \emph{arXiv:2110.06267 [cs, math]}, 2021.

\bibitem[Dixit et~al.(2021)Dixit, Ahmadi, and Burdick]{Dixit2021}
A.~Dixit, M.~Ahmadi, and J.~W. Burdick.
\newblock Risk-{{Sensitive}} motion planning using entropic value-at-risk.
\newblock In \emph{2021 {{European Control Conference}} ({{ECC}})}, pages
  1726--1732, 2021.

\bibitem[Dowson et~al.(2021)Dowson, Morton, and Pagnoncelli]{Dowson2021}
O.~Dowson, D.~P. Morton, and B.~K. Pagnoncelli.
\newblock Multistage stochastic programs with the entropic risk measure.
\newblock \emph{Preprint in Optimization Online}, 2021.

\bibitem[Eriksson and Dimitrakakis(2020)]{Eriksson2020}
H.~Eriksson and C.~Dimitrakakis.
\newblock {Epistemic risk-sensitive reinforcement learning}.
\newblock \emph{European Symposium on Artificial Neural Networks, Computational
  Intelligence and Machine Learning}, 2020.

\bibitem[Fei et~al.(2020)Fei, Yang, Chen, Wang, and Xie]{Fei2020}
Y.~Fei, Z.~Yang, Y.~Chen, Z.~Wang, and Q.~Xie.
\newblock {Risk-sensitive reinforcement learning: Near-optimal risk-sample
  tradeoff in regret}.
\newblock \emph{arXiv}, 2020.

\bibitem[F{\"{o}}llmer and Knispel(2011)]{Follmer2011}
H.~F{\"{o}}llmer and T.~Knispel.
\newblock Entropic risk measures: Coherence vs. convexity, model ambiguity and
  robust large deviations.
\newblock \emph{Stochastics and Dynamics}, 2011.

\bibitem[F{\"{o}}llmer and Schied(2002)]{Follmer2002}
H.~F{\"{o}}llmer and A.~Schied.
\newblock {Convex measures of risk and trading constraints}.
\newblock \emph{Finance and Stochastics}, 2002.

\bibitem[Follmer and Schied(2004)]{Follmer2004}
H.~Follmer and A.~Schied.
\newblock \emph{{Stochastic Finance: An Introduction in Discrete Time}}.
\newblock Walter de Gruyter, 2004.

\bibitem[Frittelli and Gianin(2004)]{Frittelli2004}
M.~Frittelli and E.~R. Gianin.
\newblock Dynamic convex risk measure.
\newblock \emph{Risk measures for the 21\textsuperscript{st} century}, 2004.

\bibitem[Frittelli and {Rosazza Gianin}(2002)]{Frittelli2002}
M.~Frittelli and E.~{Rosazza Gianin}.
\newblock {Putting order in risk measures}.
\newblock \emph{Journal of Banking and Finance (JBF)}, 2002.

\bibitem[Geist et~al.(2019)Geist, Scherrer, and Pietquin]{Geist2019}
M.~Geist, B.~Scherrer, and O.~Pietquin.
\newblock A theory of regularized {{Markov}} decision processes.
\newblock In \emph{International {{Conference}} on {{Machine Learning}}
  ({{ICML}})}, 2019.

\bibitem[{Grand-Clement} and Kroer(2021)]{Grand-Clement2021b}
J.~{Grand-Clement} and C.~Kroer.
\newblock First-order methods for {{Wasserstein}} distributionally robust
  {{MDPs}}.
\newblock In \emph{International {{Conference}} of {{Machine Learning}}
  ({{ICML}})}, 2021.

\bibitem[Ho et~al.(2021)Ho, Petrik, and Wiesemann]{Ho2021a}
C.~P. Ho, M.~Petrik, and W.~Wiesemann.
\newblock Partial policy iteration for l1-robust markov decision processes.
\newblock \emph{Journal of Machine Learning Research}, 22\penalty0
  (275):\penalty0 1--46, 2021.

\bibitem[Iancu et~al.(2015)Iancu, Petrik, and Subramanian]{Iancu2015a}
D.~A. Iancu, M.~Petrik, and D.~Subramanian.
\newblock {Tight approximations of dynamic risk measures}.
\newblock \emph{Mathematics of Operations Research}, 40\penalty0 (3):\penalty0
  655--682, 2015.

\bibitem[Iyengar(2005)]{Iyengar2005}
G.~N. Iyengar.
\newblock {Robust dynamic programming}.
\newblock \emph{Mathematics of Operations Research}, 2005.

\bibitem[Javed et~al.(2021)Javed, Brown, Sharma, Zhu, Balakrishna, Petrik,
  Dragan, and Goldberg]{Javed2021}
Z.~Javed, D.~Brown, S.~Sharma, J.~Zhu, A.~Balakrishna, M.~Petrik, A.~Dragan,
  and K.~Goldberg.
\newblock Policy gradient {{Bayesian}} robust optimization for imitation
  learning.
\newblock In \emph{International {{Conference}} on {{Machine Learning}}
  ({{ICML}})}, 2021.

\bibitem[Kupper and Schachermayer(2006)]{Kupper2006}
M.~Kupper and W.~Schachermayer.
\newblock {Representation results for law invariant time consistent functions}.
\newblock \emph{Mathematics and Financial Economics}, 16\penalty0 (2):\penalty0
  419--441, 2006.

\bibitem[Lobo et~al.(2021)Lobo, Ghavamzadeh, and Petrik]{Lobo2021}
E.~A. Lobo, M.~Ghavamzadeh, and M.~Petrik.
\newblock Soft-robust algorithms for batch reinforcement learning.
\newblock \emph{Arxiv}, 2021.

\bibitem[Mankowitz et~al.(2019)Mankowitz, Levine, Jeong, Shi, Kay, Abdolmaleki,
  Springenberg, Mann, Hester, and Riedmiller]{Mankowitz2019}
D.~J. Mankowitz, N.~Levine, R.~Jeong, Y.~Shi, J.~Kay, A.~Abdolmaleki, J.~T.
  Springenberg, T.~Mann, T.~Hester, and M.~Riedmiller.
\newblock Robust reinforcement learning for continuous control with model
  misspecification, 2019.

\bibitem[Massart(2003)]{Massart2003}
P.~Massart.
\newblock \emph{Concentration Inequalities and Model Selection}.
\newblock {Springer}, 2003.

\bibitem[Nass et~al.(2020)Nass, Belousov, and Peters]{Nass2020}
D.~Nass, B.~Belousov, and J.~Peters.
\newblock Entropic risk measure in policy search.
\newblock \emph{Investment Management and Financial Innovations}, 2020.

\bibitem[Neu et~al.(2017)Neu, Jonsson, and G{\'{o}}mez]{neu2017unified}
G.~Neu, A.~Jonsson, and V.~G{\'{o}}mez.
\newblock A unified view of entropy-regularized {Markov} decision processes.
\newblock \emph{Arxiv}, 2017.

\bibitem[Nilim and Ghaoui(2005)]{Nilim2005}
A.~Nilim and L.~E. Ghaoui.
\newblock {Robust control of {Markov} decision processes with uncertain
  transition matrices}.
\newblock \emph{Operations Research}, 53\penalty0 (5):\penalty0 780--798, 2005.

\bibitem[Osogami(2011)]{Osogami2011}
T.~Osogami.
\newblock {Iterated risk measures for risk-sensitive {Markov} decision
  processes with discounted}.
\newblock In \emph{Uncertainty in Artificial Intelligence}, 2011.

\bibitem[Osogami(2012)]{Osogami2012}
T.~Osogami.
\newblock {Robustness and risk-sensitivity in Markov decision processes}.
\newblock \emph{Advances in Neural Information Processing Systems}, 1:\penalty0
  233--241, 2012.

\bibitem[Petrik and Subramanian(2012)]{Petrik2012b}
M.~Petrik and D.~Subramanian.
\newblock {An approximate solution method for large risk-averse {Markov}
  decision processes}.
\newblock In \emph{Uncertainty in Artificial Intelligence (UAI)}, 2012.

\bibitem[Petrik and Subramanian(2014)]{Petrik2014}
M.~Petrik and D.~Subramanian.
\newblock {RAAM : The benefits of robustness in approximating aggregated {MDP}s
  in reinforcement learning}.
\newblock In \emph{Neural Information Processing Systems (NIPS)}, 2014.

\bibitem[Pflug and Pichler(2016)]{Pflug2016}
G.~C. Pflug and A.~Pichler.
\newblock Time-consistent decisions and temporal decomposition of coherent risk
  functionals.
\newblock \emph{Mathematics of Operations Research}, 41\penalty0 (2):\penalty0
  682--699, 2016.

\bibitem[Pflug and Ruszczy{\'{n}}ski(2005)]{Pflug2005}
G.~C. Pflug and A.~Ruszczy{\'{n}}ski.
\newblock {Measuring risk for income streams}.
\newblock \emph{Computational Optimization and Applications}, 2005.

\bibitem[Puterman(2005)]{Puterman2005}
M.~L. Puterman.
\newblock \emph{{Markov} decision processes: discrete stochastic dynamic
  programming}.
\newblock John Wiley \& Sons, 2005.

\bibitem[Riedel(2004)]{Riedel2004}
F.~Riedel.
\newblock {Dynamic coherent risk measures}.
\newblock \emph{Stochastic processes and their applications}, 2004.

\bibitem[Ross and Pek{\"o}z(2007)]{Ross2007}
S.~M. Ross and E.~A. Pek{\"o}z.
\newblock \emph{A Second Course in Probability}.
\newblock ProbabilityBookstore.com, {Boston}, 2007.

\bibitem[Russel and Petrik(2019)]{Russel2019beyond}
R.~H. Russel and M.~Petrik.
\newblock Beyond confidence regions: Tight bayesian ambiguity sets for robust
  mdps.
\newblock \emph{Advances in Neural Information Processing Systems}, 2019.

\bibitem[Shapiro et~al.(2014)Shapiro, Dentcheva, and Ruszczynski]{Shapiro2014}
A.~Shapiro, D.~Dentcheva, and A.~Ruszczynski.
\newblock \emph{Lectures on stochastic programming: Modeling and theory}.
\newblock SIAM, 2014.

\bibitem[Steimle et~al.(2021)Steimle, Kaufman, and Denton]{Steimle2021a}
L.~N. Steimle, D.~L. Kaufman, and B.~T. Denton.
\newblock Multi-model {{Markov}} decision processes.
\newblock \emph{IISE Transactions}, Forthcoming, 2021.

\bibitem[Tamar et~al.(2015)Tamar, Chow, Ghavamzadeh, and Mannor]{Tamar2015}
A.~Tamar, Y.~Chow, M.~Ghavamzadeh, and S.~Mannor.
\newblock Policy gradient for coherent risk measures.
\newblock In \emph{Neural Information Processing Systems}, 2015.

\bibitem[Wiesemann et~al.(2013)Wiesemann, Kuhn, and Rustem]{Wiesemann2013}
W.~Wiesemann, D.~Kuhn, and B.~Rustem.
\newblock {Robust {Markov} decision processes}.
\newblock \emph{Mathematics of Operations Research}, 2013.

\bibitem[Xu and Mannor(2012)]{Xu2012}
H.~Xu and S.~Mannor.
\newblock {Distributionally robust {Markov} decision processes}.
\newblock \emph{Mathematics of Operations Research}, 2012.

\end{thebibliography}


\appendix

\section{Proofs of \cref{sec:preliminaries}}
\label{app:sec:prelim} 

The following proposition states a simple, but important property of the expectation operator which plays a crucial role in formulating the dynamic programs. The property is known under several different names, including \emph{the tower property}, \emph{the law of total expectation}, and \emph{the law of iterated expectations}.
\begin{proposition}[Tower Property, e.g.,~Proposition~3.4 in~\cite{Ross2007}] \label{lem:tower-exp}
Any two random variables $X_1,X_2\in\mathbb X$ satisfy that
\[\E[X_1] \;=\; \E\left[\E[X_1 \mid X_2]\right]~.\]
\end{proposition}

A convenient way to represent ERM is to use its \emph{certainty equivalent} form. This form relates the risk measure to the popular expected utility framework for decision-making~\cite{Ben-Tal2005}. In the expected utility framework, one prefers a lottery (or a random reward) $X_1\in \mathbb{X}$ over $X_2 \in \mathbb{X}$ if and only if
\[
 \E[u(X_1)] \ge \E[u(X_2)]~,
\]
for some increasing \emph{utility function} $u\colon \Real \to  \Real$.

The expected utility $\E[u(X)]$ is difficult to interpret because its units are incompatible with $X$. A more interpretable characterization of the expected utility is to use the \emph{certainty equivalent} $z\in \Real$, which is defined as the certain quantity that achieves the same expected utility as $X$:
\begin{equation}\label{eq:ce-def}
  \E[u(z)] = \E[u(X)], \quad \text{and therefore}, \quad  z=u^{-1}(\E[u(X)])~.
\end{equation}
Algebraic manipulation from \eqref{eq:ce-def} then shows that ERM for any $X\in \mathbb{X}$ can be represented as the certainty equivalent
 \begin{equation}\label{eq:erm-ce}
   \erm^{\alpha}[X] \;=\; u^{-1}(\E[u(X)])~,
\end{equation}
for the utility function $u:\Real \to \Real$ (see definition 2.1 in \cite{Ben-Tal2005}) defined as
\[
  u(x) = \alpha^{-1} - \alpha^{-1} \cdot \exp {-\alpha \cdot x}~.
\]
Because the function $u$ is strictly increasing, its inverse $u^{-1}\colon  \Real \to \Real$ exists and equals to
\[
  u^{-1}(z) = - \alpha^{-1} \cdot \log (1- \alpha \cdot z)~.
\]

\begin{proof}[Proof of \cref{thm:tower-erm}]
The property is trivially true when $\alpha = 0$ from \cref{lem:tower-exp} since $\erm^0= \E$. The property then follows by algebraic manipulation for $\alpha > 0$ using the certainty equivalent representation in~\eqref{eq:erm-ce} as
\begin{align*}
\erm^{\alpha}\big[\erm^{\alpha}[X_1 \mid X_2]\big] &= \erm^{\alpha}\left[u^{-1} \left(   \E[u(X_1) \mid X_2] \right)\right] \\
&= u^{-1}\left(\E\left[ u\left( u^{-1} \left(   \E[u(X_1) \mid X_2] \right) \right) \right] \right) \\
&= u^{-1}(\E \left[ \E[u(X_1) \mid X_2] \right]) \\
&= u^{-1}( \E[u(X_1)]) && \text{~\cref{lem:tower-exp}}\\
&= \erm^{\alpha}[X_1]~.
\end{align*}
\end{proof}

\section{Proofs of \cref{sec:rasr-erm}}
\label{app:sec:ERM}

\begin{algorithm} 
    \KwIn{Horizon $T < \infty$, risk level $\alpha > 0$, terminal value $v_T(s),\;\forall s\in\states$}
    \KwOut{Optimal value $(v_t\opt)_{t=0}^{T}$ and policy $(\pi_t\opt)_{t=0}^{T-1}$}
    Initialize $v\opt_{T}(s) \gets v'(s), \; \forall s\in \states$ \;
    \For{$t = T-1, \dots ,0$}{ 
      Update $v_t\opt$ using~\eqref{eq:v-erm-opt} and $\pi_t\opt$ using~\eqref{eq:pol-greedy}\;
    }
    \Return $v\opt , \pi\opt$ \;
    \caption{VI for finite-horizon RASR-ERM}    \label{alg:RASR_VI}
\end{algorithm}

\begin{proof}[Proof of \cref{thm:pos-quasi-homogen}]
The property is trivially true for $c = 0$ or $\alpha = 0$ because $\erm^0 = \E$ and $\erm^{\alpha}[0] = 0$. Then, for $c > 0$ and $\alpha >0$, the property follows by rearranging the terms as
\begin{align*}
  \erm^{\alpha \cdot c}[X] &= -\frac{1}{\alpha c}\log\big(\E[e^{-\alpha \cdot c \cdot X}]\big) \\
  c \cdot \erm^{\alpha \cdot c}[X] &= -\frac{1}{\alpha}\log\big(\E[e^{-\alpha \cdot c \cdot X}]\big) && \text{Multiply by $c$}\\
&= \erm^{\alpha}[c\cdot X].
\end{align*}
\end{proof}

In order to prove the correctness of the dynamic program formulation for the ERM objective, we need to formalize the soft-robust uncertainty model more rigorously. For the purpose of this discussion, we assume some fixed state $\hat{s}\in \states$, action $\hat{a}\in \actions$, and a time step $\hat{t} \in 0, \dots , T-1$. To streamline the notation, we drop the subscripts of $f$, $P$, and $\bar{p}$ for $\hat{s}$, $\hat{a}$, and $\hat{t}$ in the subsequent discussion.

Now we formalize the nested distribution function, in which the random model $P$ governs the transition function of the random next state. Define a probability space $(\Omega \times \mathcal{S}, 2^{\Omega \times \mathcal{S}}, \hat{f})$, which combines the uncertainty over the model $\omega$ and the next state $s' \in \states$ at time $\hat{t}$ when taking an action $\hat{a}$ in a state $\hat{s}$. Define a 2-step filtration $\mathcal{F}_0 = 2^{\Omega}$ and $\mathcal{F}_1 = 2^{\Omega \times \mathcal{S}}$. The first step represents the model $P$ and the second step represents the choice of state $S$. 

The random variable $S\colon (s,\omega) \mapsto s$ is measurable in $\mathcal{F}_1$ and represents the state transition. For the probability function $\hat{f}$ to be consistent with model probabilities $f$ defined in \cref{sec:preliminaries}, the function $\hat{f}$ must satisfy that 
\begin{equation}\label{eq:f-fhat}
 f(\omega) = \sum_{s\in \states} \hat{f}(\omega, s), \quad \forall \omega \in \Omega~.  
\end{equation}
The vector-valued random variable $P\colon \Omega \to \Delta^S$ is measurable in $\mathcal{F}_0$ and represents the uncertain transition function. It is defined as
\[
  P(\omega , s)  = P_{\hat{t}}(\cdot \mid  \hat{s},\hat{a})~,
\]
for each $s \in \states$ and $\omega \in \Omega$. The random variable is independent of the component $s$ in the probability space. Finally, since $P$ governs the distribution of $S$, the function $\hat{f}$ must also satisfy that 
\begin{equation} \label{eq:P-definition-simple}
\hat{f}(\omega ,s') = f(\omega)\cdot P(\omega)_{s'}, \quad  \forall \omega \in \Omega , s'\in \states ~. 
\end{equation} 
That this, this equality ensures that $S \sim P$.

Our goal now is to replace any random variable $X \colon \Omega \times \states \to  \Real$  defined on the entire filtration by a simpler random variable $\tilde{X} \colon \states \to \Real$ measurable that has the same ERM value. 
The random variable $\tilde{X}$ is defined on a probability space $(\states, 2^{\states}, \bar{p})$ with
\begin{equation} \label{eq:pbar-definition}
  \bar{p}(s) = \sum_{\omega \in \Omega } \hat{f}(\omega, s) = \sum_{\omega \in \Omega } f(\omega ) \cdot P(\omega )_s, \quad  \forall s\in \states ~.  
\end{equation}
Intuitively, the value $\bar{p}$ represents the mean transition probability of the uncertain $P$. The second equality in \eqref{eq:pbar-definition} follows immediately from \eqref{eq:P-definition-simple}. The following lemma shows that the ERM values of $X$ and $\tilde{X}$ coincide. 
\begin{lemma} \label{lem:double-expectation}
Suppose that random variables $P$, $X$, and $\tilde{X}$ are defined as above and that $X$ is independent of $\omega$:
\[
 X(\omega_1, s) = X(\omega_2, s), \quad \forall s\in \states, \forall \omega_1, \omega_2 \in \Omega ~.  
\]
Then, for any $\alpha \ge 0$, the entropic risk measure satisfies that
\begin{equation} \label{eq:double-expectation}
\erm^{\alpha }[X] = \erm^{\alpha }[\tilde{X}] = \erm^{\alpha }[X \mid  P = \bar{p}]~. 
\end{equation}
\end{lemma}
\begin{proof}
We prove the property using the certainty equivalent representation of ERM in \eqref{eq:erm-ce}. For the utility function $u$ in \eqref{eq:erm-ce}, we get that
\begin{equation} \label{eq:double-exp-der}
\begin{aligned}
   \E[u(x)] &= \sum_{s\in \states } \sum_{\omega \in \Omega } u(X(s,\omega )) \hat{f}(s, \omega) =
   \sum_{s\in \states } u(X(s)) \left(  \sum_{\omega \in \Omega} \hat{f}(s,\omega) \right) \\
   &\stackrel{\text{Eq.~\eqref{eq:P-definition-simple}}}{=} \sum_{s\in \states } u(X(s)) \left(  \sum_{\omega \in \Omega} P(\omega)_s \cdot f(\omega) \right) \\
   &\stackrel{\text{Eq.~\eqref{eq:pbar-definition}}}{=} \sum_{s\in \states } u(X(s)) \left(  \sum_{\omega \in \Omega} \bar{p}(\omega)_s  \right)  = \E[u(\tilde{X})]~.
\end{aligned}
\end{equation}
Because $X$ is independent of $\omega$, the derivation above abbreviates it as $X(s) = X(\omega, s)$ for an arbitrary $\omega\in \Omega$. Equality~\eqref{eq:double-expectation} then follows from~\eqref{eq:erm-ce} and~\eqref{eq:double-exp-der} as  
\[
   \erm^{\alpha}[X] = u^{-1}\left(\E[u(X)]\right) =u^{-1}\left(\E[u(\tilde{X})]\right) = \erm^{\alpha}[\tilde{X}]~.
\]
\end{proof}

Although \cref{lem:double-expectation} is stated for ERM, one could generalize it to CVaR and certain other risk measures that admit an \emph{optimized certainty equivalent} representation~\cite{Ben-Tal2005}.

\begin{proof}[Proof of \cref{thm:dynamic-program}]
\journal{We also need to elaborate on how the uncertain model $P$ can be decomposed to be inside of the Bellman update. Because the objective is conditioned on the entire sequence $(P_t)_{t=0}^{T}$, while here we are conditioning based on $P_t$}
The proof is divided into two steps, proving it for $v^{\pi}$ first and for $v\opt$  second.

We prove the result for $v^{\pi}$ for any fixed $\pi\in \Pi_{MR}$ by backward induction on $t$ from $t=T$ to $t=0$. In particular, we show that any $v^{\pi}_t$ that satisfies~\eqref{eq:v-erm-pi} must equal to its definition in~\eqref{eq:erm-rasr-v} and is, therefore, also unique. The base case of the induction with $t = T$ is trivial because $v_T(s) = 0$ for each $s\in \states$ by definition.

To prove the inductive step, we assume that any $v^{\pi}$ that satisfies the Bellman optimality condition in~\eqref{eq:v-erm-pi} must equal to~\eqref{eq:erm-rasr-v} for $t+1 \le T$. We then prove that~\eqref{eq:v-erm-pi} implies~\eqref{eq:erm-rasr-v} also for $t$ and each $s_t\in \states$.

Next, we first express the equality in~\eqref{eq:v-erm-pi} in terms of a random model $P_t$ and introduce random variables $A_t \colon \actions  \times \Omega \times \states \to \actions$ and $S''\colon  \actions \times \Omega \times \states \to \states$ where $A_t(a,\omega,s) = a$ represents the actions taken and $S''(a,\omega, s) = s$ represents the next state for each $a\in \actions$, $\omega \in \Omega$, $s\in \states$, and the randomized action is distributed as $A_t \sim \pi_t(s_t)$. The random variables $A_t$ and $S''$ are defined over the joint probability space of $\actions \times \Omega \times \states$ with a probability mass function
\[
  \hat{f}(a, \omega, s) =  \pi_t(a \mid s_t) \cdot  f_t(\omega) \cdot P_t(\omega)(s \mid  s_t, a)~.
\]
The function $\hat{f}$ is the joint probability over $a$, $\omega$, and $s$. Recall also that $S' \sim \bar{p}_t(\cdot|s,A)$. The Bellman equation in \eqref{eq:v-erm-pi} can be reformulated as
\begin{align}
\nonumber
  v_t^{\pi}(s_t)
  &= \erm^{\alpha \cdot \gamma^t} \left[r(s_t,A_t) + \gamma\cdot v_{t+1}^{\pi}(S') \right] \\
\nonumber
  &= \erm^{\alpha \cdot \gamma^t}\left[\erm^{\alpha \cdot \gamma^t} \left[r(s_t,A_t) + \gamma\cdot v_{t+1}^{\pi}(S') \mid  A_t \right] \right] && \text{\cref{thm:tower-erm}} \\
\nonumber
  &= \erm^{\alpha \cdot \gamma^t}\left[\erm^{\alpha \cdot \gamma^t} \left[r(s_t,A_t) + \gamma\cdot v_{t+1}^{\pi}(S'') \mid  A_t \right] \right] && \text{\cref{lem:double-expectation}}\\
  \label{eq:proof-dp-eqdp}
  &= \erm^{\alpha \cdot \gamma^t} \left[r(s_t,A_t) + \gamma\cdot v_{t+1}^{\pi}(S'') \right]~.  && \text{\cref{thm:tower-erm}}
\end{align}
Above, \cref{lem:double-expectation} is applied to random variables \(  \tilde{X} = \erm^{\alpha \cdot \gamma^t} \left[r(s_t,A_t) + \gamma\cdot v_{t+1}^{\pi}(S') \mid  A_t \right] \) and $X = \erm^{\alpha \cdot \gamma^t} \left[r(s_t,A_t) + \gamma\cdot v_{t+1}^{\pi}(S'') \mid  A_t \right]$.

Next, by the induction hypothesis we can substitute the definition of $v_{t+1}^{\pi}(S'')$ from~\eqref{eq:erm-rasr-v} into~\eqref{eq:proof-dp-eqdp} and use the positive quasi-homogeneity of ERM~(\cref{thm:pos-quasi-homogen}) to get that
\begin{align*}
  v_t^{\pi}(s_t) &= \erm^{\alpha\cdot \gamma^t} \left[r(s_t,A_t) + \gamma \cdot \erm^{\alpha \cdot \gamma^{t+1}} \left[\sum_{t'=t+1}^{T-1} \gamma^{t'-t-1} \cdot r(S_{t'},A_{t'})  \mid S_{t+1} = S'' \right] \right] \\
  &= \erm^{\alpha\cdot \gamma^t} \left[r(s_t,A_t) + \erm^{\alpha \cdot \gamma^{t}} \left[\sum_{t'=t+1}^{T-1} \gamma^{t'-t} \cdot r(S_{t'},A_{t'})  \mid S_{t+1} = S'' \right] \right]~. 
\end{align*}
Here, $S_t = s_t$, $A_{t'} \sim \pi_{t'}(S_{t'})$ for each $t' \ge t+1$. The random variables $S_{t'}, t'=t+1, \dots , T$ follow the policy $\pi$ and state transition probabilities.

Finally, using translation equivariance~(A2 in \cref{def:convex-risk}) and tower~(\cref{thm:tower-erm}) properties of ERM, we conclude that
\begin{align*}
  v_t^{\pi}(s_t)
  &= \erm^{\alpha\cdot \gamma^t} \left[r(S_t,A_t) + \erm^{\alpha \cdot \gamma^{t}} \left[\sum_{t'=t+1}^{T-1} \gamma^{t'-t} \cdot r(S_{t'},A_{t'})  \mid S_{t+1} = S'', A_t \right]  \mid S_t = s_t\right]  \\
  &= \erm^{\alpha\cdot \gamma^t} \left[ \erm^{\alpha \cdot \gamma^{t}} \left[r(S_t,A_t) +\sum_{t'=t+1}^{T-1} \gamma^{t'-t} \cdot r(S_{t'},A_{t'})  \mid S_{t+1} = S'', A_t \right]  \mid S_t = s_t\right] \\
  &= \erm^{\alpha\cdot \gamma^t} \left[ \erm^{\alpha \cdot \gamma^{t}} \left[ \sum_{t'=t}^{T-1} \gamma^{t'-t} \cdot r(S_{t'},A_{t'})  \mid S_{t+1} = S'', A_t \right]  \mid S_t = s_t\right] \\
  &= \erm^{\alpha \cdot \gamma^{t}} \left[ \sum_{t'=t}^{T-1} \gamma^{t'-t} \cdot r(S_{t'},A_{t'}) \mid S_t = s_t\right] ~.
\end{align*}
The derivation above shows that any $v^{\pi}$ that satisfies the Bellman equation in~\eqref{eq:v-erm-pi} is unique and satisfies the definition in \eqref{eq:erm-rasr-v}. 

To prove the second part of the theorem, which concerns $v\opt$, we formally define the optimal value function $v\opt$ for each $t = 0, \dots, T$ as
\begin{equation} \label{eq:erm-rasr-v-opt}
v\opt_t(s) \;=\; \max_{\pi \in \Pi_{MR}^t} \erm^{\alpha\cdot \gamma^{t}} \left[ \sum_{t'=t}^{T-1} \gamma^{t'-t} \cdot  R^{\pi}_{t'}  \mid S_t = s \right], \quad \forall s\in\mathcal S~.
\end{equation}
Here, the set $\Pi_{MR}^t$ represents the randomized Markov policies for time steps $t, \dots, T-1$.

The proof for the optimal value function $v\opt$ now proceeds by backward induction analogously to the proof of \eqref{eq:v-erm-pi} with the difference that it incorporates the optimization over actions. As before, the base case with $t = T$ is trivial because $v\opt_T(s) = 0$ for each $s\in \states$.

To prove the inductive step, assume that \eqref{eq:v-erm-opt} implies \eqref{eq:erm-rasr-v-opt} for $t+1 \le T$ and show the implication for $t$ and each $s_t\in \states$. Let $S''$ and $A_t$ be defined as in the first part of the proof, except that $A_t$ need not be distributed according to $\pi_t$. Then from \cref{lem:double-expectation}, we formulate the Bellman equation in terms of $S''$ that is jointly random over $\Omega$ and $\states$ as
\begin{align*}
  v_t\opt(s_t)
  &= \max_{a\in \actions} \erm^{\alpha \cdot \gamma^t} \left[r(s_t,A_t) + \gamma\cdot v_{t+1}\opt (S') \mid  A_t = a \right] \\
  &= \max_{a\in \actions}\erm^{\alpha \cdot \gamma^t} \left[r(s_t,A_t) + \gamma\cdot v_{t+1}\opt (S'') \mid  A_t = a \right] ~.
\end{align*}
Then, \cref{proof:deterministic_action} shows that the maximization over actions can be replaced by a maximization over randomized policies which consider distributions over actions as
\begin{align}
\nonumber
  v\opt(s_t) &= \max_{a\in \actions}\erm^{\alpha \cdot \gamma^t} \left[r(s_t,A_t) + \gamma\cdot v_{t+1}\opt (S'') \mid  A_t = a \right]  \\
  \label{eq:dp-opt}
&= \max_{d\in \Delta^A}\erm^{\alpha \cdot \gamma^t} \left[r(s_t,A_t) + \gamma\cdot v_{t+1}\opt(S'') \mid  A_t \sim d \right]~.
\end{align}

Next, we can substitute the definition of $v_{t+1}\opt(S'')$ from~\eqref{eq:erm-rasr-v-opt} into \eqref{eq:dp-opt} by the induction hypothesis and use the positive quasi-homogeneity of ERM~(\cref{thm:pos-quasi-homogen}) to get that
\begin{align*}
  v_t\opt(s_t) &= \max_{d\in \Delta^A} \erm^{\alpha\cdot \gamma^t} \left[r(S_t,A_t) + \gamma   \max_{\pi \in \Pi_{MR}^{t+1}} \erm^{\alpha \cdot \gamma^{t+1}} \left[\sum_{t'=t+1}^{T-1} \gamma^{t'-t-1}  r(S_{t'},A_{t'})  \mid S_{t+1} = S'' \right]  \right] \\
  &= \max_{d\in \Delta^A} \erm^{\alpha\cdot \gamma^t} \left[r(S_t,A_t) +  \max_{\pi \in \Pi_{MR}^{t+1}} \erm^{\alpha \cdot \gamma^{t}} \left[\sum_{t'=t+1}^{T-1} \gamma^{t'-t} \cdot r(S_{t'},A_{t'})  \mid S_{t+1} = S'' \right] \right]~. 
\end{align*}
Here, $S_t = s_t$, $A_t \sim d$, and $A_{t'} \sim \pi_{t'}(S_{t'})$ for each $t' \ge t+1$ and the random variables $S_{t'}, t'=t+1, \dots , T$ are governed by $\pi_{t'}$ and transition probabilities. We can move $\gamma$ inside the maximization operator because it is non-negative.

Finally, using monotonicity~(A1 in \cref{def:convex-risk}), translation equivariance~(A2 in \cref{def:convex-risk}) and tower~(\cref{thm:tower-erm}) properties of ERM, we conclude that
\begin{align*}
  v_t\opt (s_t)
  &= \max_{d\in \Delta^A} \erm^{\alpha\cdot \gamma^t} \left[r(S_t,A_t) + \max_{\pi \in \Pi_{MR}^{t+1}}\erm^{\alpha \cdot \gamma^{t}} \left[\sum_{t'=t+1}^{T-1} \gamma^{t'-t} \cdot r(S_{t'},A_{t'})  \mid S_{t+1} = S'', A_t \right] \right]  \\
  &= \max_{d\in \Delta^A}\erm^{\alpha\cdot \gamma^t} \left[ \max_{\pi \in \Pi_{MR}^{t+1}}\erm^{\alpha \cdot \gamma^{t}} \left[r(S_t,A_t) +\sum_{t'=t+1}^{T-1} \gamma^{t'-t} \cdot r(S_{t'},A_{t'})  \mid S_{t+1} = S'', A_t \right] \right] \\
  &= \max_{d\in \Delta^A}\erm^{\alpha\cdot \gamma^t} \left[ \max_{\pi \in \Pi_{MR}^{t+1}}\erm^{\alpha \cdot \gamma^{t}} \left[ \sum_{t'=t}^{T-1} \gamma^{t'-t} \cdot r(S_{t'},A_{t'})  \mid S_{t+1} = S'', A_t \right] \right] \\
  &= \max_{\pi \in \Pi_{MR}^t}\erm^{\alpha \cdot \gamma^{t}} \left[ \sum_{t'=t}^{T-1} \gamma^{t'-t} \cdot r(S_{t'},A_{t'}) \mid S_t = s_t\right] ~.
\end{align*}
As before, $S_t = s_t$, $A_t \sim d$, and $A_{t'} \sim \pi_{t'}(S_{t'})$ for each $t' \ge t+1$.The derivation above shows that any $v\opt$ that satisfies \eqref{eq:v-erm-opt} equals to the definition in \eqref{eq:erm-rasr-v-opt} and is, therefore, unique. 
\end{proof}

\begin{proof}[Proof of \cref{thm:average-model}]
The proof follows immediately by reformulating the objective in~\eqref{eq:RASR-evar-return} as the dynamic programming equations in \cref{thm:dynamic-program}. These dynamic programming equations are identical for the dynamic programming equations for a risk-averse ERM objective with the average model $\bar{p} = (\bar{p}_t)_{t=0}^{T-1}$ as defined in \cref{thm:dynamic-program}.
\end{proof}

\begin{proof}[Proof of \cref{th:optimal_deterministic}]
Following the notation of Chapter 4 in \cite{Puterman2005}, let $\mathcal{H}_t$ be the set of all histories up to time $t$ inclusively. Let the optimal history-dependent value function be $u_t\opt \colon  \mathcal{H}_t \to \Real, t = 0, \dots , T$. The value function $u\opt = (u\opt_t)_{t=0}^{T}$ is achieved by the optimal history-dependent policy because the state and actions are finite and, therefore, the space of randomized history-dependent policies is compact.

The proof proceeds in 2 steps. First, we show that $v\opt$ attains the return of the optimal history-dependent value function:
\[
u\opt_t(h_t) = v\opt_t(s_t) \quad  \forall h_t\in \mathcal{H}_t, t = 0, \dots , T~,
\]
where $s_t$ is the $t$-th and final state in the history $h_t$. This result is a consequence of the dynamic programming formulation in \cref{thm:dynamic-program}. 

An argument analogous to the proof of Theorem 4.4.2(a) in \cite{Puterman2005} shows that $u\opt_t(h_t)$ only depends on $s_t$, which is the final state in the history $h_t$.

Using the standard backward-induction argument on $t$, assume that $u\opt_{t+1}(h_{t+1}) = v\opt_{t+1}(s_{t+1})$ holds and we prove that $u\opt_t(h_t) = v\opt_t(s_t)$. Let $d_t \in \Delta^A$ be the randomized decision-rule (action) that achieves $u\opt(h_t)$ where $s_t$ is the last state in the history $h_t$. Then applying \cref{thm:dynamic-program} applied to $u\opt$ (using history-dependent state space) and the inductive assumption we conclude that
\begin{align*}
  u\opt_t(h_t) &= \erm^{\alpha \cdot \gamma^t} \left[r(s_t,A) + \gamma\cdot u_{t+1}\opt ((h_t,A,S')) \mid  A \sim d_t\right] \\
               &= \erm^{\alpha \cdot \gamma^t} \left[r(s_t,A) + \gamma\cdot v_{t+1}\opt(S') \mid  A \sim d_t\right] \\
  &=  \max_{d\in \Delta^A} \erm^{\alpha \cdot \gamma^t} \left[r(s_t,A) + \gamma\cdot v_{t+1}\opt(S') \mid A \sim  d \right]  ~.
\end{align*}
The notation $(h_t,A,S')$ signifies a new history $h_{t+1}$ constructed by appending the action $A$ and state $S'$. The state $S'$ is a random variable distributed according to $\bar{p}(\cdot \mid s_t, A)$. Then, using \cref{app:sec:tech-lem} and the Bellman optimality condition in~\eqref{eq:v-erm-opt}, we conclude that
\begin{align*}
  u\opt_t(h_t) &=  \max_{d\in \Delta^A} \erm^{\alpha \cdot \gamma^t} \left[r(s_t,A) + \gamma\cdot v_{t+1}\opt(S') \mid A \sim  d \right]  \\
  &= v\opt_t(s_t).
\end{align*}

The second part of the proof is to show that $\pi\opt$ achieves the optimal value function $v\opt$:
\[
v_t^{\pi\opt}(s) = v_t\opt(s), \quad \forall s\in \states, t = 0, \dots, T~.
\]
This result follows using the standard backward induction argument and algebraic manipulation from \cref{thm:dynamic-program}. The derivation relies on the fact that $\actions$ is finite and the maximum in \eqref{eq:v-erm-opt} exists (and is achieved). The fact that $\pi\opt$ is greedy to $v\opt$ is trivial from the definition of a greedy policy.  
\end{proof}

The following lemma plays an important role in bounding the difference between ERM and the expectation. This result serves to bound the error of replacing the risk-averse value function by a risk-neutral value function in \cref{alg:rasr-vi-inf}.
\begin{lemma} \label{lem:bound-approximation}
Let $X \in \mathbb{X}$ be a bounded random variable such that $x_{\min} \le X \le x_{\max}$ a.s. Then, for any risk level $\alpha > 0$, we have $\E[X] - \epsilon(\alpha) \le \erm^{\alpha}[X] \le \E[X]$, where
\[
\epsilon(\alpha) \;=\; 8^{-1} \cdot \alpha \cdot (x_{\max}-x_{\min})^2 \,.
\]
The gap vanishes with a decreasing risk:~$\lim_{\alpha \to 0} \epsilon(\alpha) = 0$.
\end{lemma}
\begin{proof}
To streamline the notation, let $X = \mathfrak{R}^{\pi}_T$ for any policy $\pi$ which is bounded between $x_{\min}$ and $x_{\max}$. Also, recall that the Hoeffding's lemma shows that for any $\forall \lambda \in \real$ \cite{Massart2003,Boucheron2013}
\[
    \E[e^{\lambda X}] \leq e^{\lambda \E[x] + \frac{\lambda^2 (x_{\max}-x_{\min})^2}{8}}~.
  \]
Applying $\log $ to both sides of the inequality above gives
\[
    \log\big(\E[e^{\lambda X}]\big) \leq \lambda \E[x] + \frac{\lambda^2 (x_{\max}-x_{\min})^2}{8} ~.
\]
Then substitute $\lambda = -\alpha$ into the equation above to by algebraic manipulation that
\begin{align*}
    \log\big(\E[e^{-\alpha X}]\big) & \leq -\alpha\cdot  \E[x] + \frac{\alpha^2 \cdot (x_{\max}-x_{\min})^2}{8}\\
    -\frac{1}{\alpha}\log\big(\E[e^{-\alpha X}]\big) & \geq  \E[x] - \frac{\alpha (x_{\max}-x_{\min})^2}{8}\\
    \E[x] - \frac{\alpha (x_{\max}-x_{\min})^2}{8} & \leq \erm^\alpha[X]~.
\end{align*}
Therefore,  $\E[X] - \epsilon(\alpha) \le \erm^{\alpha}[X]$ where $\epsilon(\alpha) = 8^{-1}\alpha (x_{\max}-x_{\min})^2$.

The second inequality $\erm^{\alpha}[X] \leq \E[X]$ follows immediately from the dual representation of ERM used in the proof of \cref{proof:deterministic_action}. 
\end{proof}

\begin{proof}[Proof of \cref{thm:approx-error}]
The main idea of the proof is to lower-bound the value function $v^{\hat{\pi}\opt}$ of the policy $\hat{\pi}\opt$ using the value function $v^{\infty}$ of the optimal risk-neutral policy. Recall that \cref{lem:bound-approximation} bounds the error between the risk-neutral and ERM value function of any policy $\pi$ and any $t = 0, \dots$ as
\begin{equation} \label{eq:basic-bound}
    0 \le  v^{\infty}_{\pi} - v^{\pi }_t \le \frac{\alpha \cdot \gamma^t \cdot (\triangle r)^2}{8\cdot (1-\gamma)^2}~.
\end{equation}
The symbol $v^{\infty}_{\pi}$ denotes the ordinary risk-neutral ($\erm^0$) $\gamma$-discounted infinite-horizon value function of the policy $\pi$. Note that this value function is stationary. The left-hand side of the equation above holds because $\E$ is an upper bound on the ERM.

As the first step of the proof, we bound the error at time $T'$ as follows. Consider any state $s\in \mathcal{S}$, then:
\begin{align*}
  v\opt_{T'} (s)- v_{T'}^{\hat{\pi}\opt}(s) &\le  v\opt_{T'}(s) - v^{\infty}_{\hat{\pi}\opt}(s) +  \frac{\alpha\cdot \gamma^{T'} \cdot (\triangle r)^2}{8\cdot (1-\gamma)^2} && \text{from r.h.s of~\eqref{eq:basic-bound}} \\
  &\le  v^{\infty}_{\pi\opt}(s) - v^{\infty}_{\hat{\pi}\opt}(s) +  \frac{\alpha\cdot \gamma^{T'} \cdot (\triangle r)^2}{8\cdot (1-\gamma)^2} && \text{from l.h.s. of~\eqref{eq:basic-bound}} \\
  &\le  \frac{\alpha \cdot \gamma^{T'} \cdot (\triangle r)^2}{8\cdot (1-\gamma)^2} && \text{from } \hat{\pi}\opt\in \arg \max_{\pi \in \Pi} v^{\infty}_{\pi }(s) ~.
\end{align*}

As the second step of the proof, we construct an approximation $u_t\in \Real^S, t = 0, \dots ,T'$ of the value function $v_t^{\hat{\pi}\opt}$ for $t = 0, \dots, T'-1$ and all $s\in \states$ as:
\begin{align*}
  u_{T'}(s) &= v^{\infty}_{\hat{\pi}\opt} -  \frac{\alpha \cdot \gamma^{T'} \cdot (\triangle r)^2}{8\cdot (1-\gamma)^2} \\
  u_t(s) &= \max_{a\in \actions} \erm^{t \cdot \gamma^t} \left[ r(s,a) + \gamma \cdot u_{t+1}(S'_{t+1,a}) \right] \\
  &= \erm^{t \cdot \gamma^t} \left[ r(s,\hat{\pi}\opt(s)) + \gamma \cdot u_{t+1}(S'_{t+1,\hat{\pi}\opt(s)}) \right] ~,
\end{align*}
where $S'_{t+1,a}$ denotes the random variable that represents the state that follows $s$ at time $t+1$ after taking an action $a$. The last equality holds from the definition of $\hat{\pi}\opt_t$ being greedy to $u_t$; subtracting a constant from all states does not change the greedy policy. The function $u_t$ is constructed to be a lower bound on $v_t^{\hat{\pi}\opt}$ and at the same time be a value such that $\hat{\pi}\opt$ is greedy to it. 

From~\eqref{eq:basic-bound}, we have that $v^{\pi}_{T'}(s) \ge u_{T'}(s)$ for all $s\in \states$. Then, assuming $v^{\pi}_{t+1}(s) \ge u_{t+1}(s)$ for all $s\in\states$, we can use backward induction on $t$ to show that
\begin{align*}
  v^{\hat{\pi}\opt}_t(s) - u_t(s)  &= \erm^{t \cdot \gamma^t} \left[ r(s,\hat{\pi}\opt_t(s)) + \gamma \cdot v^{\hat{\pi}\opt}_{t+1}(S'_{t+1,\hat{\pi}\opt_t(s)}) \right] - \\
                                &\qquad  - \erm^{t \cdot \gamma^t} \left[ r(s,\hat{\pi}\opt_t(s)) + \gamma \cdot u_{t+1}(S'_{t+1,\hat{\pi}\opt_t(s)}) \right] \\
     &\stackrel{\text{(a)}}{=} \erm^{t \cdot \gamma^t} \left[ \gamma \cdot v^{\hat{\pi}\opt}_{t+1}(S'_{t+1,\hat{\pi}\opt_t(s)}) \right] -  \erm^{t \cdot \gamma^t} \left[ \gamma \cdot u_{t+1}(S'_{t+1,\hat{\pi}\opt_t(s)}) \right]  \\
     &\stackrel{\text{(b)}}{=} \gamma \cdot\left(\erm^{t \cdot \gamma^{t+1}} \left[  v^{\hat{\pi}\opt}_{t+1}(S'_{t+1,\hat{\pi}\opt_t(s)}) \right] -  \erm^{t \cdot \gamma^{t+1}} \left[  u_{t+1}(S'_{t+1,\hat{\pi}\opt_t(s)}) \right] \right) \\
     &\stackrel{\text{(c)}}{\ge} 0~.
\end{align*}
The equality (a) is shown by subtracting the constant reward from both terms which can be done because ERM is translation equivariant. The equality (b) follows from the positive quasi-homogeneity in \cref{thm:pos-quasi-homogen}, and (c) follows from the monotonicity of ERM.

As the third step of the proof, we show for each $s\in \states$ and $t = 0, \dots, T'$ that
\begin{equation}\label{eq:v-u-bound}
 v_t\opt (s) - u_t(s) \le \gamma^{T' - t}  \cdot \frac{\alpha \cdot \gamma^{T'} \cdot (\triangle r)^2}{8\cdot (1-\gamma)^2}~.
\end{equation}
The inequality~\eqref{eq:v-u-bound} holds for $t=T'$ from~\eqref{eq:basic-bound} and the construction of $u_{T'}$. To prove~\eqref{eq:v-u-bound} by induction, assume it holds for $t+1$. Then for each $s\in \states$:
\begin{align}
  \notag v_t\opt (s) -  u_t(s) &\stackrel{\text{(a)}}{=} \erm^{t \cdot \gamma^t} \left[ r(s,\pi\opt_t (s)) + \gamma \cdot v\opt_{t+1}(S'_{t+1,\pi\opt_t(s)}) \right] - \\
  \notag &\qquad  - \erm^{t \cdot \gamma^t} \left[ r(s,\hat{\pi}\opt_t(s)) + \gamma \cdot u_{t+1}(S'_{t+1,\hat{\pi}\opt_t(s)}) \right] \\
  \notag &\stackrel{\text{(b)}}{=} \erm^{t \cdot \gamma^t} \left[ r(s,\pi\opt_t (s)) + \gamma \cdot v\opt_{t+1}(S'_{t+1,\pi\opt_t(s)}) \right] - \\
  \notag &\qquad  - \erm^{t \cdot \gamma^t} \left[ r(s,\pi\opt_t(s)) + \gamma \cdot u_{t+1}(S'_{t+1,\pi\opt_t(s)}) \right] \\
    \notag &\stackrel{\text{(c)}}{=} \erm^{t \cdot \gamma^t} \left[ \gamma \cdot v^{\hat{\pi}\opt}_{t+1}(S'_{t+1,\pi\opt_t(s)}) \right] -  \erm^{t \cdot \gamma^t} \left[ \gamma \cdot u_{t+1}(S'_{t+1,\pi\opt_t(s)}) \right]  \\
     \label{eq:ubound-rhs} &\stackrel{\text{(d)}}{=} \gamma \cdot\left(\erm^{t \cdot \gamma^{t+1}} \left[  v^{\pi\opt}_{t+1}(S'_{t+1,\pi\opt(s)}) \right] -  \erm^{t \cdot \gamma^{t+1}} \left[  u_{t+1}(S'_{t+1,\pi\opt(s)}) \right] \right)
\end{align}
The identity (a) is true by the definition of $v\opt_t$ and $u_t$, (b) follows from $\hat{\pi}\opt$ being greedy to $u$, (c) follows by subtracting the constant reward from both terms which can be done because ERM is translation equivariant. Finally, the equality (d) follows from the positive quasi-homogeneity in \cref{thm:pos-quasi-homogen}. Then, from the inductive assumption we get the desired inequality from the monotonicity and translation equivariance of ERM by bounding the terms in~\eqref{eq:ubound-rhs} above as:
\begin{align}
  \nonumber
  v_{t+1}^{\pi\opt}(s) - u_{t+1}(s) &\le \gamma^{T' - t - 1}  \cdot \frac{\alpha \cdot \gamma^{T'} \cdot (\triangle r)^2}{8\cdot (1-\gamma)^2}  \qquad \forall  s \in \states \\
  \label{eq:approx-proof-bound}
  \erm^{t \cdot \gamma^{t+1}}[v_{t+1}^{\pi\opt}(S)] - \erm^{t \cdot \gamma^{t+1}}[u_{t+1}(S)] &\le \gamma^{T' - t - 1}  \cdot \frac{\alpha \cdot \gamma^{T'} \cdot (\triangle r)^2}{8\cdot (1-\gamma)^2} \\
  \nonumber
\gamma \cdot (\erm^{t \cdot \gamma^{t+1}} [v_{t+1}^{\pi\opt}(S)] - \erm^{t \cdot \gamma^{t+1}} [u_{t+1}(S)] ) &\le \gamma^{T' - t }  \cdot \frac{\alpha \cdot \gamma^{T'} \cdot (\triangle r)^2}{8\cdot (1-\gamma)^2} ~.
\end{align}
Because inequality in~\eqref{eq:approx-proof-bound} holds for any random variables $S$, we can substitute $S = S'_{t+1,\pi\opt _{t+1}(s)}$ from~\eqref{eq:ubound-rhs}. The substitution then proves the bound on $u_t$.

The theorem then follows from the properties established above as
\begin{align*}
  \erm^{\alpha} \big[ \mathfrak{R}_{\infty}^{\pi\opt} \mid P=\bar{p} \big] - 
\erm^{\alpha} \big[ \mathfrak{R}_{\infty}^{\hat{\pi}\opt} \mid P=\bar{p} \big]
  &=  v_0\opt (s_0) - v^{\hat{\pi}\opt}_0(s_0) \le  v_0\opt (s_0) - u_0 \\
  &\le  \frac{\alpha \cdot \gamma^{2\cdot T'} \cdot (\triangle r)^2}{8\cdot (1-\gamma)^2}~.
\end{align*}
\end{proof}


\section{Proofs of \cref{sec:rasr-evar}}
\label{app:sec:EVaR}

\begin{figure}
  \centering
 \includegraphics[width=0.45\textwidth]{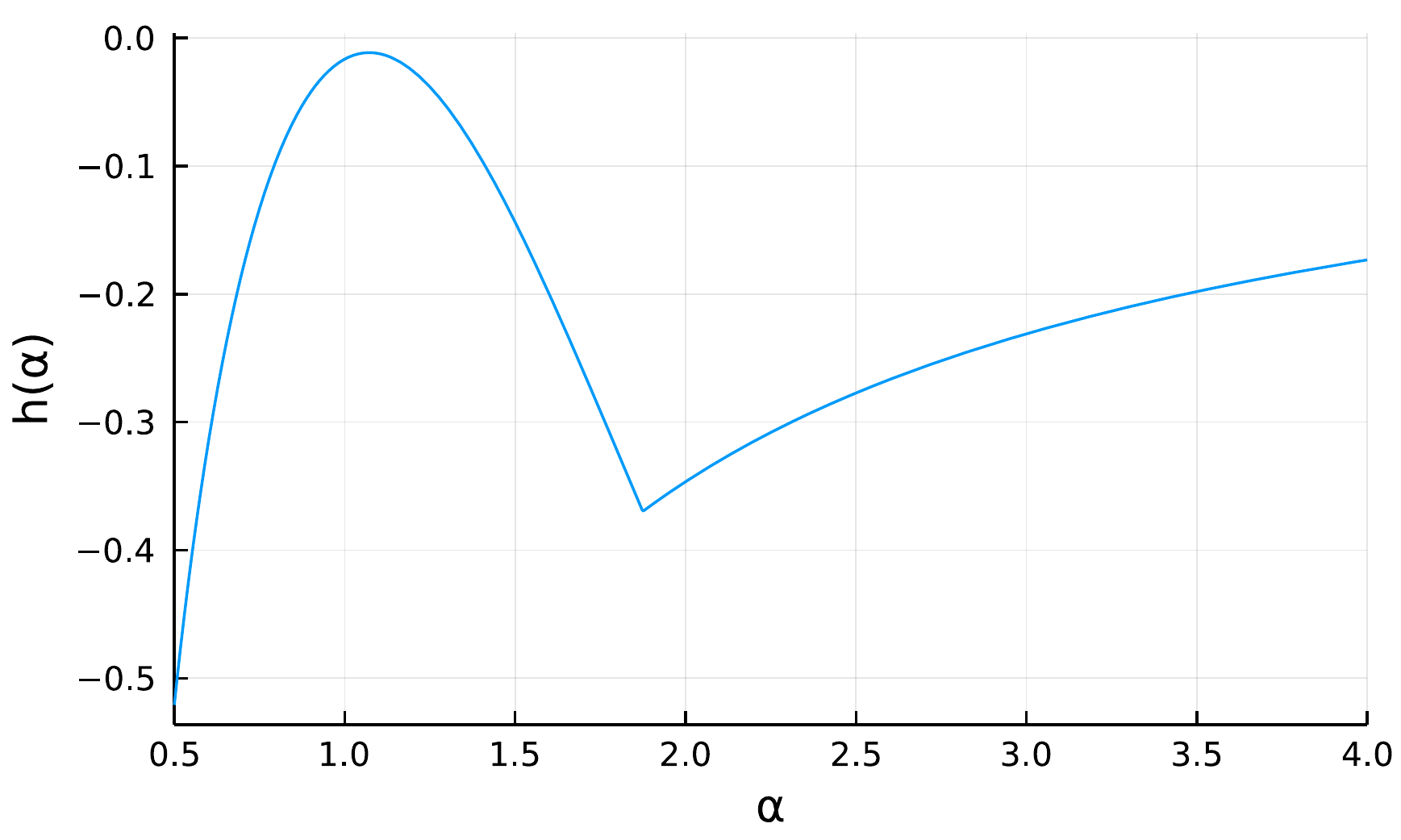}
 \hspace{0.08\textwidth}
 \includegraphics[width=0.45\textwidth]{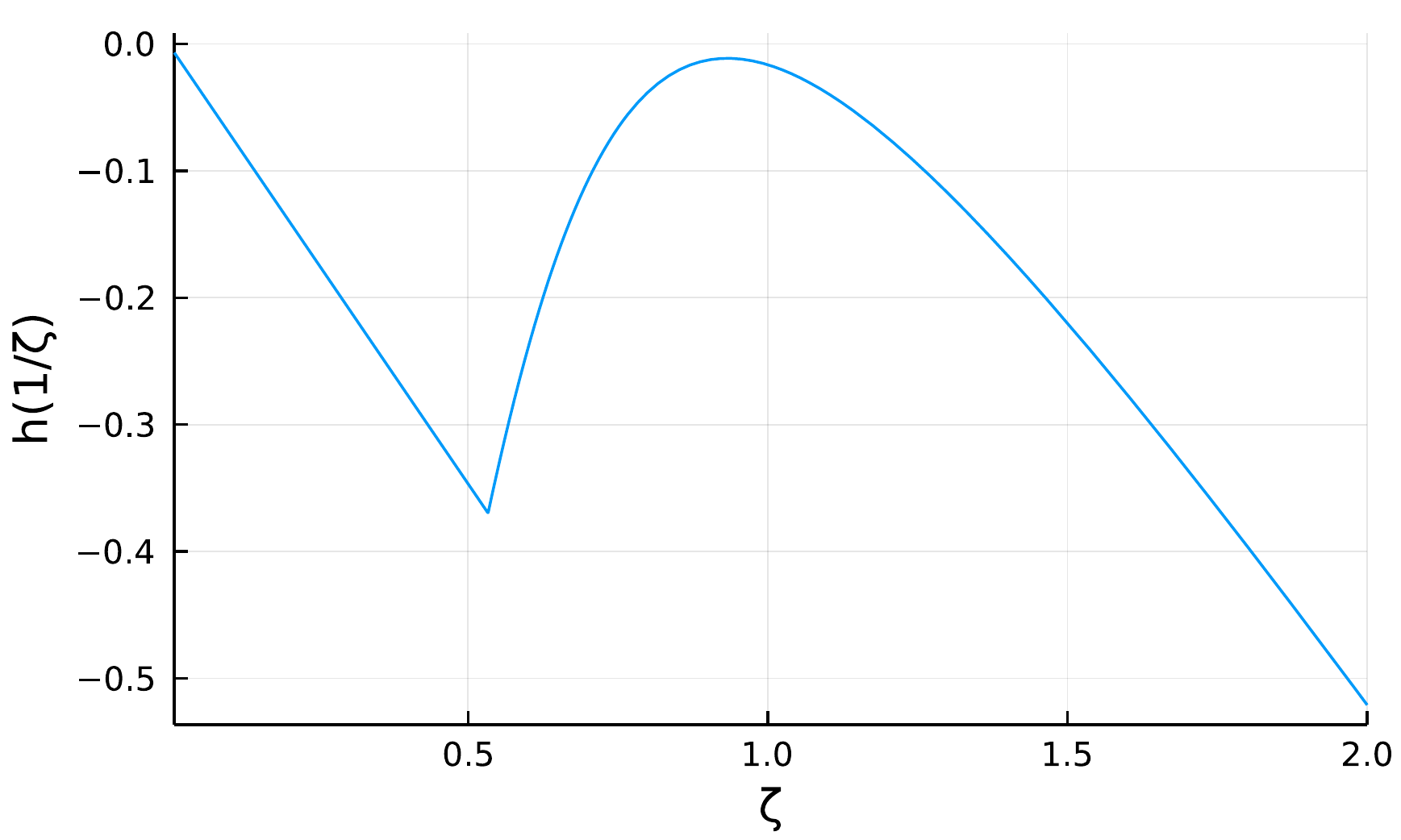}
  \caption{Plots of $h(\alpha)$ (left) and $h(\zeta^{-1})$ (right) for $\beta=0.5$, which are used in the proof of \cref{prop:non-concave}.}
  \label{fig:h-counter-proof}
\end{figure}

Recall that the function $h\colon \Real \to \Real$ is defined in \eqref{eq:rasr-evar-ref} as
\[
 h(\alpha) \;=\; \max_{\pi\in\Pi_{MR}}\, \left(\erm^\alpha[\mathfrak{R}_{T}^\pi] + \alpha^{-1} \cdot \log(1-\beta) \right) ~.
\]
Because solving RASR-EVaR reduces to computing $\max_{\alpha \ge 0} h(\alpha)$, it would be ideal if it were concave or at least quasi-concave. When $|\Pi_{MR}| = 1$, it is known that the function $\zeta \mapsto h(\zeta^{-1})$ is concave~\cite{Ahmadi-Javid2012} and $h$ is quasi-concave. Unfortunately, the following proposition shows that $h$ is not necessarily quasi-concave when $|\Pi_{MR}| > 1$ which precludes the use of more efficient optimization methods when solving $\max_{\alpha \ge 0} h(\alpha)$.
\begin{proposition}\label{prop:non-concave}
There exists an MDP and $\beta \in [0,1)$ such that the function $h\colon  \Real \to \Real$ defined in \eqref{eq:rasr-evar-ref} is neither concave nor convex either in $\alpha$ or $\alpha^{-1}$.
\end{proposition}
\begin{proof}
We show the property by constructing a counter-example for which the function $h$ is not concave. Consider a risk-averse problem (no epistemic uncertainty) with an MDP with states $\states  = \{s_0, s_1, s_2, s_3\}$, actions $\actions = \{a_1, a_2\}$, and a finite-horizon objective with $\gamma = 1$ and $T = 2$. The initial state is $s_0$. Define the MDP's parameters as
\begin{align*}
    p(\cdot  \mid s_0, a_1) &= [0, 0, 1, 0] \\
    p(\cdot  \mid s_0, a_2) &= [0, 0.02, 0, 0.98] \\
    r(\cdot, a_i) &= [-2, 0, 1], \quad  i \in  \{1,2\}~.
\end{align*}
The transition probabilities from $s_1, s_2, s_3$ are irrelevant because of the limited horizon $T=2$. Setting $\beta = 0.5$, one can easily verify that $h(\alpha)$ and $h(\zeta^{-1})$ are not quasi-concave by computing $h(1)$, $h(2)$, and $h(4)$. \Cref{fig:h-counter-proof} shows the plots of both $h(\alpha)$ and $h(\zeta^{-1})$.  
\end{proof}

\begin{proof}[Proof of \cref{thm:equivalence-evar-erm}]
We prove the contra-positive: If $\pi \opt$ is not optimal policy in RASR-ERM for all $\alpha > 0$, then $\pi \opt$ is not an optimal solution to RASR-EVaR. Assume $\pi \opt$ is not an optimal policy for all $\alpha >0$, and $\pi_\alpha$ is an optimal policy for RASR-ERM$^\alpha$,
\begin{align*}
    \erm^\alpha\left[X^{\pi \opt}\right] & < \erm^\alpha\left[X^{\pi_\alpha}\right] \qquad , \forall ~ \alpha > 0 \\
    \sup_{\alpha >0} \left\{\erm^\alpha[X^{\pi \opt}] + \frac{\log(1-\beta)}{\alpha}\right\} 
    & < \sup_{\alpha >0} \left\{\erm^\alpha[X^{\pi_\alpha}] + \frac{\log(1-\beta)}{\alpha}\right\}  \\
    \evar^\beta[X] & < \sup_{\alpha >0} \left\{\erm^\alpha\left[X^{\pi_\alpha}\right] + \frac{\log(1-\beta)}{\alpha}\right\}
\end{align*}
We prove that if $\pi \opt$ is not optimal policy in RASR-ERM for all $\alpha > 0$, then $\pi \opt$ is not an optimal solution to RASR-EVaR. With contra-positive we prove that if $\pi\opt$ is an optimal solution to RASR-EVaR$^\beta$ in~\eqref{eq:RASR-evar-return} then there exists a value $\alpha\opt$ such that $\pi\opt$ is optimal in RASR-ERM with risk level $\alpha = \alpha\opt$.
\end{proof}

\begin{proof}[Proof of \cref{cor:evar-markov}]
\cref{thm:equivalence-evar-erm} shows that the optimal policy $\pi \opt$ for $\evar^\beta[X^{\pi \opt}]$ implies there exists $\alpha\opt$ such that $\erm^{\alpha \opt}[X^{\pi \opt}]$ is optimal in RASR-ERM and \cref{th:optimal_deterministic} shows that there exists a Markov deterministic time-dependent optimal policy $\pi \opt = (\pi_t\opt)_{t=0}^{T-1}\in\Pi_{MD}$ for~\eqref{eq:RASR-erm-return}. Therefore there exists a Markov deterministic time-dependent optimal policy $\pi \opt$ which optimizes the EVaR objective in~\eqref{eq:RASR-evar-return}.

The second part of the corollary can be shown as follows. For any policy $\pi \in \Pi_{MR}$, the RASR-EVaR objective in~\eqref{eq:RASR-erm-return} can be written as 
\begin{align*}
  \evar^{\beta} \left[  \mathfrak{R}_{T}^\pi \right] &= \sup_{\alpha >0} \left( \erm^\alpha[\mathfrak{R}_{T}^\pi] + \frac{\log(1-\beta)}{\alpha} \right) \\
&= \sup_{\alpha >0} \left( \erm^\alpha[\mathfrak{R}_{T}^\pi \mid P=\bar{p}] + \frac{\log(1-\beta)}{\alpha} \right) \\
 &=  \evar^{\beta} \left[ \mathfrak{R}_{T}^\pi \mid P=\bar{p} \right]\;.
\end{align*}
\end{proof}

The following lemma plays an important role in bounding the error introduced by discretizing the risk level $\alpha$ in \cref{alg:ERM_EVAR}. 
\begin{lemma}\label{lem:evar-discretize}
Suppose that $\alpha\opt$ attains the maximum in~\eqref{eq:rasr-evar-ref} then $\alpha_0 \ge \alpha\opt \ge \alpha_K$, and $h(\hat{\alpha}) \ge h(\alpha_k)$ for $k = 0, \dots , K$ and some $\alpha_0 \ge \dots \ge \alpha_K$. Then
\begin{equation*}
h(\alpha\opt) - h(\hat{\alpha}) \le \log(1-\beta) \max_{k\in 0, \dots, K-1}  \left(\alpha_k^{-1} - \alpha_{k+1}^{-1}\right)  ~.
\end{equation*}
Also, $h(\alpha\opt) - h(\hat{\alpha}) \le - \log(1-\beta) \alpha_0^{-1}$ when $\alpha\opt > \alpha_0$.
\end{lemma}
\begin{proof}
Given that the optimal risk $\alpha_{l+1} \leq \alpha \opt \leq \alpha_{l}$, where $\alpha_l$ and $\alpha_{l+1}$ are in the set of ERM levels $\Lambda$ we have computed. We can bound 
\begin{align*}
    \evar^\beta(X) - \max_{\alpha \in \Lambda} \left\{\erm^\alpha[X] + \frac{\log(1-\beta)}{\alpha}\right\} \leq  \log(1-\beta) \left(\frac{1}{\alpha_l} - \frac{1}{\alpha_{l+1}}\right)
\end{align*}
By the monotonicity property of ERM we get
\begin{align*}
    \erm^{\alpha_l}[X^\pi_{\alpha_{l+1} }] &\leq \erm^{\alpha_l}[X^\pi_{\alpha_{l}}] 
           \leq \erm^{\alpha{\opt}}[X^\pi_{\alpha_{l}}] \\
    &\leq \erm^{\alpha\opt}[X^\pi_{\alpha \opt}] 
    \leq \erm^{\alpha_{l+1}}[X^\pi_{\alpha \opt}] \leq \erm^{\alpha_{l+1}}[X^\pi_{\alpha_{l+1}}]
\end{align*}
where $X^\pi_{\alpha}$ refers to the total discounted reward distribution deploying the optimal policy of $\erm^\alpha$. On the other hand, 
\begin{align*}
    &\frac{\log(1-\beta)}{\alpha_{l+1}} \leq \frac{\log(1-\beta)}{\alpha\opt}  \leq \frac{\log(1-\beta)}{\alpha_{l}} ~.
\end{align*}
Using the inequality above, we can conclude that
\begin{align*}
    \erm^{\alpha_l}[X^{\pi}_{\alpha_{l}}] + \frac{\log(1-\beta)}{\alpha_{l+1}}
    \leq \erm^{\alpha {\opt}}[X^\pi_{\alpha \opt}] + \frac{\log(1-\beta)}{\alpha \opt} 
    \leq \erm^{\alpha_{l+1}}[X^{\pi}_{\alpha_{l+1}}] + \frac{\log(1-\beta)}{\alpha_{l}}
\end{align*}
Therefore,
\begin{align*}
    &\evar^\beta(X) - \max_{\alpha \in \Lambda} \left\{\erm^\alpha[X] + \frac{\log(1-\beta)}{\alpha}\right\} \\
    &\leq \erm^{\alpha {\opt}}[X^\pi_{\alpha \opt}] + \frac{\log(1-\beta)}{\alpha \opt}  - \max_{\alpha \in \{\alpha_{l+1}\} }\left\{\erm^\alpha[X^\pi_\alpha] + \frac{\log(1-\beta)}{\alpha}\right\} \\
    &\leq \frac{\log(1-\beta)}{\alpha_{l}} - \frac{\log(1-\beta)}{\alpha_{l+1}}  \\
    &= \log(1-\beta) \left(\frac{1}{\alpha_l} - \frac{1}{\alpha_{l+1}}\right)~.
\end{align*}
Now we relax the assumption to $\alpha \opt \in [\alpha_0,\alpha_K]$, and conclude that
\[
h(\alpha\opt) - h(\hat{\alpha}) \le \max_{k = 0, \dots , K-1} \left\{ \log(1-\beta) \left(\frac{1}{\alpha_k} - \frac{1}{\alpha_{k+1}}\right) ~\right\}~.
\]

The last part of the theorem can be proved as follows. Given an arbitrary error tolerance $\delta$, $\beta$ and $\alpha_k$ \cref{col:opt_spacing} shows that we can set $\alpha_{k+1} = ( \frac{1}{\alpha_k} - \frac{\delta }{\log(1-\beta)})^{-1}$ such that $h(\alpha\opt) - h(\hat{\alpha}) \le \delta$. Moreover, for $\alpha \opt > \alpha_0$, given $\alpha_{0}$ and $\beta$ the error $h(\alpha\opt) - h(\hat{\alpha})  \le - \frac{\log(1-\beta)}{\alpha_{0}}$.
\end{proof}

\begin{proof}[Proof of \cref{thm:rasr-evar-bound}]
Assume $\alpha\opt \in \arg \max_{\alpha >0} h(\alpha)$ be the $\alpha$ that achieves the optimality in the definition $\evar^{\beta} \big[ X \big] = \sup_{\alpha >0} h(\alpha)$. The supremum is achieved whenever $\beta>0$ since then there exists an optimal $\alpha\opt >0$. Then, $h(\alpha\opt) \geq h(\alpha\opt + \epsilon)$ for any $\epsilon > 0$
\begin{gather*}
h(\alpha\opt) \geq h(\alpha\opt + \epsilon)\\
\erm^{\alpha\opt}[X] + \frac{\log(1-\beta)}{\alpha\opt} \geq \erm^{\alpha\opt+\epsilon}[X] + \frac{\log(1-\beta)}{ \alpha\opt+\epsilon} \\
\erm^{\alpha\opt}[X] -\erm^{\alpha\opt+\epsilon}[X] \geq  \frac{\log(1-\beta)}{ \alpha\opt+\epsilon} - \frac{\log(1-\beta)}{\alpha\opt} \\
\frac{(\triangle r)^2}{8(1-\gamma)^2} \geq \frac{d (\erm^{\alpha\opt}[X])}{d \alpha\opt} \geq \log(1-\beta) \frac{d (\alpha \opt)^{-1}}{d \alpha \opt} \\
\frac{(\triangle r)^2}{8(1-\gamma)^2} \geq -\log(1-\beta)(\alpha \opt)^{-2}\\
(\alpha \opt)^{2} \geq -\log(1-\beta) \frac{8(1-\gamma)^2}{(\triangle r)^2} \\
\alpha \opt \geq \sqrt{-8\log(1-\beta)} \frac{(1-\gamma)}{(\triangle r)}
\end{gather*}
We let $\alpha_0 \to \infty$. Then, to achieve the desired bound, we need to choose the number of points $K$ such that $\sqrt{-8 \log(1-\beta)} \frac{1-\gamma}{\triangle r} \geq \alpha_K$. Then, following the construction in \cref{col:opt_spacing}, we get that $\alpha_K = \frac{-\log(1-\beta)}{K \delta}$ and 
\begin{gather*}
\sqrt{-8 \log(1-\beta)} \frac{1-\gamma}{\triangle r} \geq \frac{-\log(1-\beta)}{K \delta}\\
K \geq \sqrt{\frac{-\log(1-\beta)}{8 }} \frac{\triangle r}{(1-\gamma) \delta }~.
\end{gather*}
We conclude the proof with \cref{lem:evar-discretize} since $\alpha_0 \geq \alpha \opt \geq \alpha_K$.
\end{proof}


\section{Technical Lemmas}
\label{app:sec:tech-lem}

The following lemma helps to show that a deterministic policy can attain the same return as a randomized policy when the objective is an ERM. This result is not surprising and derives from the fact that the $\erm^{\alpha}[X] \le \max_{\omega \in \Omega} X$ for any random variable $X\in \mathbb{X}$ defined over a finite probability space.
\begin{lemma}  \label{proof:deterministic_action}
Let $X\colon \Omega \to \mathcal{A}$ be a random variable defined over a finite action set $\mathcal{A}$ and let $g\colon  \mathcal{A}\to \Real$ be a function defined for each action. Then, for any $\alpha \ge 0$, we have that
\[
    \max_{a \in \actions} g(a) \;=\;  \max_{d \in \Delta^X} \erm^{\alpha} \left[ g(X) \mid  X \sim d \right]~.
  \]
\end{lemma}
\begin{proof}
We first prove the inequality $\max_{a \in \actions} g(a) \le \max_{d \in \Delta^A} \erm^{\alpha} \left[ g(X) \mid  X \sim d \right]$. Let $a\opt \in \arg\max_{a\in \actions } g(a)$ be an optimal action. We now construct a policy $\bar{d} \in \Delta^A$ as $\bar{d}(a\opt) = 1$ and $\bar{d}(a) = 0$ for each $a\in \actions \setminus \{a\opt \}$. Substituting $\bar{d}$ in the definition of ERM yields that
\begin{equation} \label{eq:erm-singleton}
\begin{aligned}
\erm^{\alpha} [g(X) \mid X \sim \bar{d}] &= - \alpha^{-1} \cdot \log \Bigl( \E\left[\exp{-\alpha\cdot  g(X)} \mid  X \sim \bar{d} \right] \Bigr)  \\
  &= - \alpha^{-1} \cdot \log \Bigl( \exp{-\alpha\cdot  g(-\alpha \cdot a\opt)} \Bigr) \\
  &= g(a\opt)~.
\end{aligned}
\end{equation}
Using~\eqref{eq:erm-singleton} and the fact that $\bar{d}$ is a valid probability distribution in $\Delta^{A}$, we obtain the desired inequality as
\[
  \max_{a \in \actions} \; g(a) = g(a\opt) = \erm^{\alpha} [g(X) \mid X \sim \bar{d}] \le  \max_{d \in \Delta^A} \; \erm^{\alpha} [g(X) \mid X \sim d] ~.
\]

To prove the converse inequality $\max_{a \in \actions} g(a) \ge \max_{d \in \Delta^{A}} \erm^{\alpha} \left[ g(X) \mid  X \sim d\right]$, let $d\opt \in \arg\max_{d \in \Delta^{A}} \erm^{\alpha} \left[ g(X) \mid X \sim d \right]$ be an optimal distribution. It will be convenient to use the dual representation of $\erm^\alpha[g(X) \mid X \sim d]$ which takes the following form for any $d \in \Delta^{A}$~(e.g.,~\cite{Ahmadi-Javid2012}):
\[
  \erm^\alpha[g(X) \mid X \sim d] \; =\; \inf_{\bar{d} \in \Delta^A}\left\{ \E[g(X) \mid X \sim \bar{d}] + \frac{1}{\alpha} \KL(\bar{d} \| d) \;\vert\; \bar{d} \ll d\right\}~,
\]
where $\KL$ is the KL-divergence and $\ll$ denotes absolute continuity. Using this dual representation, we get the following upper bound on $\erm^{\alpha} \left[ g(X)  \mid  X \sim d\opt\right]$:
\begin{align*}
\erm^\alpha[g(X) \mid X \sim d\opt] &= \inf_{\bar{d} \in \Delta^A}\left\{ \E[g(X) \mid X \sim \bar{d}] + \frac{1}{\alpha} \KL(\bar{d} \| d\opt)\;\vert\; \bar{d} \ll d\opt \right\}  \\
	&\le \E[g(X) \mid X \sim d\opt] + \frac{1}{\alpha} \KL(d\opt  \| d\opt)  \\ 
        &\stackrel{\text{(a)}}{=} \E[g(X) \mid X \sim d\opt] \\
  &\stackrel{\text{(b)}}{\leq} \max_{a\in\actions} g(a)~.
\end{align*}
The equality (a) holds $\KL(d \| d) = 0$, and equality (b) follows because $A$ is finite, and, therefore, for each $d \in \Delta^{A}$:
\[
  \max_{a\in\actions} g(a) \geq \E\left[g(X) \mid  X \sim d \right]~.
\]
This proves the second desired inequality since $d\opt \in \Delta^{A}$ and concludes the proof. 
\end{proof}

The proof techniques used to show \cref{proof:deterministic_action} would also work to to establish an equivalent property for CVaR and other coherent risk measures $\psi\colon  \mathbb{X}\to \Real$ that satisfy that $\psi(X) \le \E[X]$ for any random variable $X$.

\begin{corollary} \label{col:opt_spacing}
Given an arbitrary error tolerance $\delta$, $\beta$ and $\alpha_k$ we construct $\alpha_{k+1}$ as $\alpha_{k+1} = ( \frac{1}{\alpha_k} - \frac{\delta }{\log(1-\beta)})^{-1}$ such that $\alpha_k \ge \alpha_{k+1} > 0$ and 
\[
 \log(1-\beta) \left(\frac{1}{\alpha_k} - \frac{1}{\alpha_{k+1}}\right) = \delta ~.
\]
Moreover, given $\alpha_{k+1}$ and $\beta$ the error $\delta  \le - \frac{\log(1-\beta)}{\alpha_{k+1}}$.
\end{corollary}
\begin{proof}
Let $\alpha_{k+1} = c \cdot  \alpha_k$ for $c \in (0,1)$, we can derive the following
\begin{align*}
    \log(1-\beta) \left(\frac{1}{\alpha_k} - \frac{1}{\alpha_{k+1}}\right) &= \delta   \\ 
    \log(1-\beta) \left(\frac{c - 1}{c \cdot \alpha_k}\right) &= \delta   \\ 
    c - 1 &= \frac{\delta \cdot c \cdot \alpha_k}{\log(1-\beta)}\\  
        c \cdot \alpha_k  \left( \frac{1}{\alpha_k} - \frac{\delta }{\log(1-\beta)}\right) &=  1  \\ 
        c \cdot \alpha_k   &=  \left( \frac{1}{\alpha_k} - \frac{\delta }{\log(1-\beta)}\right)^{-1}  \\ 
        \alpha_{k+1} &=  \left( \frac{1}{\alpha_k} - \frac{\delta }{\log(1-\beta)}\right)^{-1}
\end{align*}
Let $\alpha_k$ approach $\infty$, the reverse implication of $\alpha_{k+1}$ to the error $\delta$ can be verified to be
\begin{align*}
    \alpha_{k+1} =  \left( \frac{1}{\alpha_k} - \frac{\delta }{\log(1-\beta)}\right)^{-1} &\leq \lim_{\alpha_k \to \infty} \left( \frac{1}{\alpha_k} - \frac{\delta }{\log(1-\beta)}\right)^{-1} =-\frac{\log(1-\beta)}{\delta }~.
\end{align*}
Combining the inequalities above, we conclude that 
\begin{align*}
    \delta & \leq - \frac{\log(1-\beta)}{\alpha_{k+1}}~.
\end{align*}
\end{proof}


\section{Risk Measures}
\label{secapp:risk-measure-prop}

Consider a probability space $(\Omega,\mathcal F,P)$. 
Let $\mathbb X:\Omega\rightarrow\mathbb R$ be a space of $\mathcal F$-measurable functions (space of $\mathcal F$-measurable random variables). 

\begin{definition}[Risk Measure]
\label{def:risk}
A risk measure is a function $\psi:\mathbb X\rightarrow\mathbb R$ that maps a random variable $X\in\mathbb X$ to real numbers. 
\end{definition}

\begin{definition}[Coherent Risk Measure] \label{def:coherent-risk}
A risk measure $\psi$ is \emph{coherent} if it satisfies the following four axioms~\citep{Artzner1999}:
\begin{itemize}
\item[A1.] Monotonicity: \hfill $X_1 \leq X_2 \; (a.s.) \Longrightarrow \psi[X_1] \leq \psi[X_2],\;\;\forall X_1,X_2\in\mathbb X$.
\item[A2.] Translation Equivariance: \hfill $\psi[c+X] = c + \psi[X],\;\;\forall c\in\mathbb R,\;\forall X\in\mathbb X$. 
\item[A3.] (a) Sub-Additivity: \hfill $\psi[X_1 + X_2] \leq \psi[X_1] + \psi[X_2],\;\;\forall X_1,X_2\in\mathbb X$. \\
(b) Super-Additivity: \hfill $\psi[X_1 + X_2] \geq \psi[X_1] + \psi[X_2],\;\;\forall X_1,X_2\in\mathbb X$.
\item[A4.] Positive Homogeneity: \hfill $\psi[cX] = c\psi[X],\;\;\forall c\in\mathbb {R}_+,\;\forall X\in\mathbb X$.
\end{itemize}
Axioms A3(a) and A3(b) are used for cost minimization and reward maximization, respectively. 
\end{definition}
Common coherent risk measures include  $\cvar^\beta$, and $\evar^\beta$ that we define them below. Convex risk measures are a more general class of risk measures (than coherent risk measures) and are defined as

\begin{definition}[Convex Risk Measure] \label{def:convex-risk}
A \emph{convex} risk measure $\psi$ satisfies axioms A1 and A2 (in \cref{def:coherent-risk}) and replaces axioms A3 and A4 with the following axiom:
\begin{itemize}
\item[A5.] (a) Convexity: \hfill $\psi\big[cX_1 + (1 - c)X_2\big] \leq c\psi[X_1] + (1 - c)\psi[X_2],\;\;\forall c\in[0,1],\;\forall X_1,X_2\in\mathbb X$. \\
(b) Concavity: \hfill $\psi\big[cX_1 + (1 - c)X_2\big] \geq c\psi[X_1] + (1 - c)\psi[X_2],\;\;\forall c\in[0,1],\;\forall X_1,X_2\in\mathbb X$.
\end{itemize}
Axioms A5(a) and A5(b) are used for cost minimization and reward maximization, respectively. 
\end{definition}

Every coherent risk measure is a convex risk measure but a convex risk measure may not be coherent. In particular, if a risk measure satisfies A3 (sub or super additivity) and A4 (positive homogeneity), then it satisfies A5 (convexity), but the reverse is not always true. For instance, Entropic risk measure (ERM), defined below, is convex but not incoherent. 


\subsection{Value-at-Risk}
\label{subsec:VaR}

For a random variable $X\in\mathbb X$, its value-at-risk with confidence level $\beta$, denoted by $\var^\beta[X]$, is the $(1-\beta)$-quantile of $X$:
\begin{equation*}
\var^\beta[X] = \inf_{x\in\mathbb R}\big\{F_X(x) > 1 - \beta\big\} = F_X^{-1}(1-\beta),\quad\beta\in [0,1),
\end{equation*}
where $F_X$ is the cumulative distribution function of $X$.


\subsection{Conditional Value-at-Risk}
\label{subsec:CVaR}

For a random variable $X\in\mathbb X$, its conditional value-at-risk with confidence level $\beta$, denoted by $\cvar^\beta[X]$, is defined as the expectation of the worst $(1-\beta)$-fraction of $X$, and can be computed as the solution of the following optimization problem:
\begin{equation*}
\cvar^\beta[X] = \inf_{\zeta\in\mathbb R}\left(\zeta - \frac{1}{1 - \beta}\cdot\mathbb E\big[(\zeta - X)_+\big]\right),\quad\beta\in [0,1).
\end{equation*}
It is easy to see that $\cvar^0[X] = \mathbb E[X]$ and $\lim_{\beta\rightarrow 1}\cvar^\beta[X] = \operatorname{ess} \inf[X]$, where the \emph{essential infimum} of $X$ is defined as $\operatorname{ess} \inf[X] = \sup_{\zeta\in\mathbb R}\big\{\mathbb P(X < \zeta)=0\big\}$.


\subsection{Entropic Risk Measure}
\label{subsec:ERM}

For a random variable $X\in\mathbb X$, its entropic risk measure with risk parameter $\alpha$, denoted by $\erm^\alpha[X]$, is defined as
\begin{equation*}
\erm^\alpha[X] = -\frac{1}{\alpha}\log\left(\mathbb E[e^{-\alpha X}]\right),\quad\alpha>0.
\end{equation*}

\paragraph{Properties of ERM:}
\begin{enumerate}
\item It is easy to see that $\lim_{\alpha\rightarrow 0}\erm^\alpha[X] = \mathbb E[X]$ and $\lim_{\alpha\rightarrow \infty}\erm^\alpha[X] =\operatorname{ess}\inf[X]$.
\item For any random variable $X\in\mathbb X$, we have $\erm^\alpha[X] = \mathbb E[X] - \frac{\alpha}{2}\vari[X] + o(\alpha)$. 
\item If $X$ is a Gaussian random variable, we have $\erm^\alpha[X] = \mathbb E[X] - \frac{\alpha}{2}\vari[X]$.
\item We have $\erm^\alpha[X_1 \mid X_2] = -\frac{1}{\alpha}\log\left(\mathbb E[e^{-\alpha X_1} \mid X_2]\right),\;\forall X_1,X_2\in\mathbb X$.
\item We have $\erm^\alpha[cX] \neq c\erm^\alpha[X]$, thus, ERM does not satisfy the axiom A4 (positive homogeneity) and is not a coherent risk measure. 
\end{enumerate}


\subsection{Entropic Value-at-Risk}
\label{subsec:EVaR}

For a random variable $X\in\mathbb X$, its entropic value-at-risk with confidence level $\beta$, denoted by $\evar^\beta[X]$, is defined as
\begin{equation*}
\evar^\beta[X] = \sup_{\alpha>0}\left(\erm^\alpha[X] + \frac{\log(1 - \beta)}{\alpha}\right),\quad\beta\in[0,1).
\end{equation*}

\paragraph{Properties of EVaR:}
\begin{enumerate}
\item EVaR with confidence level $\beta$ is the tightest possible lower-bound that can be obtained from the Chernoff ineqaulity for $\var$ and $\cvar$ with confidence level $\beta$:
\begin{equation*}
\evar^\beta[X] \leq \cvar^\beta[X] \leq \var^\beta[X].    
\end{equation*}
\item The following inequality also holds for EVaR:
\begin{equation*}
\operatorname{ess}\inf[X] \leq \evar^\beta[X] \leq \mathbb E[X].      
\end{equation*}
\item It is easy to see that $\evar^0[X] = \mathbb E[X]$ and $\lim_{\beta\rightarrow 1}\evar^\beta[X] = \operatorname{ess}\inf[X]$.
\end{enumerate}


\subsection{Properties of Risk Measures}
\label{subsec:risk-measure-prop}

\Cref{tab:risk_measure_props} summarizes some properties of convex risk measures that are desirable in RL and MDP. 

\begin{table}
  \centering
  \begin{tabular}{l|ccc}
    \toprule
    Risk measure & LI & DC & PH \\
    \midrule
    $\mathbb{E}$, Min & \cmark & \cmark & \cmark \\
    CVaR & \cmark & $\cdot$ & \cmark \\
    EVaR & \cmark & $\cdot$ & \cmark \\
    ICVaR & $\cdot$ & \cmark & \cmark \\
    ERM & \cmark & \cmark & $\cdot$ \\
    \bottomrule
    \end{tabular}
  \caption{Properties of representative risk measures.}  \label{tab:risk_measure_props}
\end{table}

A \emph{law-invariant} (LI) risk measure depends only on the total return and not on the particular sequence of individual rewards~\cite{Shapiro2014}. A \emph{dynamically-consistent} (DC), or time-consistent, risk measure satisfies the tower property~\cite{Shapiro2014} and can be optimized using a dynamic program~\cite{Cvitanic1999,Pflug2005,Riedel2004,Delbaen2006,Frittelli2004,Artzner2004,Dowson2021}. Finally, a positively-homogeneous (PH) risk measure satisfies $\psi(c\cdot X) = c \cdot \psi(X)$, for any $c \ge 0$, which is an important property in the risk-averse parameter selection and discounted setting~\cite{Artzner1999,Follmer2011}. Unfortunately, expectation ($\E[\cdot]$) and minimum (Min) are the only convex risk measures that satisfy all the desirable conditions. In \cref{tab:risk_measure_props}, ICVaR is an iterated version of CVaR~\cite{Petrik2012b,Iancu2015a}.


\section{Additional Experimental Results and Details}\label{sec:experiments-detail}



To remove the bias of the hyperparameter selection across domains algorithm, we use the same set $\Lambda$ for all domains when computing the EVaR solution. In all the numerical results in the paper, we only call \cref{alg:rasr-vi-inf} once ($K=1$) by using $\alpha = e^{10}$, $T' =(10+15)/(1-\gamma)$ without discarding any intermediate $\alpha_t$. By doing so, we have $\alpha_{0:(T'+1)} = \{e^{10}, e^{10}\gamma,e^{10}\gamma^2, ..., e^{10}\gamma^{T'},0 \} = \Lambda$ for EVaR where $e^{-15} > e^{10}\gamma^{T'} \approx 0$. This method allow us to generate each $\alpha_t$ beyond $0$ in one single value iteration. 

Furthermore, for the \cref{tab:evar_099} and \cref{fig:combined_barplot} in the main body of the paper, we sample 100,000 Monte-Carlo instances with 1,000 time horizon for each instance which take days to compute.

Just like RASR-ERM, \emph{naive} and \emph{Erik} algorithms require one to choose a risk parameter $\alpha \in (0, \infty)$. The range of $\alpha$ makes it challenging to compare  algorithms that optimize ERM with other algorithms that target VaR, CVaR, and EVaR. VaR, CVaR, and EVaR use a parameter $\beta \in [0,1]$ and can all be seen as approximations of each other. To ensure that the comparison with RASR-EVaR and other algorithms is fair, we use optimal $\alpha\opt$ computed in \cref{alg:ERM_EVAR}. 

We use a fixed seed of 1, sample only 10,000 Monte-Carlo instances, and use only 500 time horizon for each instance. The risk of return in the appendix are consistent with the paper despite generated with different Monte Carlo samples. In \cref{tab:appendix_Results}, all other baselines except \emph{Derman} perform badly in \emph{population}, and \emph{Derman} performs poorly in \emph{river-swim}. However, RASR is able to consistently mitigate risk of return when measured in all VaR, CVaR, and EVaR for all domains. Moreover, RASR was able to be computed in polynomial-time and outperform the other baseline algorithms in computation time (see \cref{tab:algorithm_time}) which makes it the most practical method available for risk-averse soft-robust RL.

\begin{table}
  \small
  \centering
\begin{minipage}{0.85\linewidth}
\small
\centering
\begin{tabular}{l|rrr}
  \toprule
 \multicolumn{1}{c|}{Method} & \multicolumn{1}{c}{RS} & \multicolumn{1}{c}{POP} & \multicolumn{1}{c}{INV} \\
  \midrule
RASR & $< 2$ & 24 & $< 7$ \\
Naive & 27 & 175 & 186 \\
Erik & 1117 & 110306 & 9977 \\
Chow & 69 & 861 & 572 \\
  \bottomrule
\end{tabular}
\caption{Time (sec) to compute each algorithm}
  \label{tab:algorithm_time}
\end{minipage}
\end{table}

\begin{table}

90\% Risk of return
\centering
\begin{tabular}[t]{l|rrr|rrr|rrr}
\hline
\multicolumn{1}{c|}{domain} & \multicolumn{3}{c|}{river-swim} & \multicolumn{3}{c|}{inventory} & \multicolumn{3}{c}{population} \\
\cline{1-1} \cline{2-4} \cline{5-7} \cline{8-10}
  & VaR & CVaR & EVaR & VaR & CVaR & EVaR & VaR & CVaR & EVaR\\
\hline
\textbf{RASR} & \textbf{50} & \textbf{50} & \textbf{50} & 327 & \textbf{319} & \textbf{310} & -623 & \textbf{-1954} & \textbf{-3920}\\
\hline
Naive & \textbf{50} & \textbf{50} & \textbf{50} & 325 & 317 & \textbf{310} & \textbf{-566} & -2014 & -4378\\
\hline
Erik & \textbf{50} & \textbf{50} & 47 & 327 & 317 & 307 & -1916 & -4090 & -5792\\
\hline
Derman & \textbf{50} & 36 & 24 & 327 & 316 & 305 & -625 & -2082 & -4364\\
\hline
RSVF & \textbf{50} & 49 & 42 & 304 & 298 & 292 & -2807 & -4881 & -6204\\
\hline
BCR & \textbf{50} & 49 & 42 & 307 & 301 & 295 & -2969 & -4985 & -6282\\
\hline
RSVI & \textbf{50} & 49 & 41 & 306 & 300 & 294 & -2646 & -4702 & -6104\\
\hline
Chow & \textbf{50} & 46 & 34 & \textbf{328} & \textbf{319} & 307 & -914 & -2126 & -4517\\
\hline
\end{tabular}

\vspace{0.1in}
95\% Risk of return
\centering
\begin{tabular}[t]{l|rrr|rrr|rrr}
\hline
\multicolumn{1}{c|}{domain} & \multicolumn{3}{c|}{river-swim} & \multicolumn{3}{c|}{inventory} & \multicolumn{3}{c}{population} \\
\cline{1-1} \cline{2-4} \cline{5-7} \cline{8-10}
  & VaR & CVaR & EVaR & VaR & CVaR & EVaR & VaR & CVaR & EVaR\\
\hline
\textbf{RASR} & \textbf{50} & \textbf{50} & \textbf{50} & 320 & 312 & \textbf{305} & -1531 & \textbf{-2948} & \textbf{-4735}\\
\hline
Naive & \textbf{50} & \textbf{50} & \textbf{50} & 318 & 311 & 304 & \textbf{-1525} & -3052 & -5285\\
\hline
Erik & \textbf{50} & 49 & 46 & 320 & 310 & 301 & -3620 & -5553 & -6739\\
\hline
Derman & 39 & 26 & 18 & 318 & 309 & 297 & -1626 & -3117 & -5277\\
\hline
RSVF & \textbf{50} & 48 & 40 & 272 & 268 & 263 & -4950 & -6465 & -7292\\
\hline
BCR & \textbf{50} & 48 & 40 & 302 & 296 & 291 & -4640 & -6258 & -7177\\
\hline
RSVI & \textbf{50} & 48 & 40 & 301 & 296 & 291 & -4314 & -6042 & -7000\\
\hline
Chow & \textbf{50} & 33 & 29 & \textbf{321} & \textbf{313} & 301 & -2305 & -3428 & -5557\\
\hline
\end{tabular}

\vspace{0.1in}
99\% Risk of return
\centering
\begin{tabular}[t]{l|rrr|rrr|rrr}
\hline
\multicolumn{1}{c|}{domain} & \multicolumn{3}{c|}{river-swim} & \multicolumn{3}{c|}{inventory} & \multicolumn{3}{c}{population} \\
\cline{1-1} \cline{2-4} \cline{5-7} \cline{8-10}
  & VaR & CVaR & EVaR & VaR & CVaR & EVaR & VaR & CVaR & EVaR\\
\hline
\textbf{RASR} & \textbf{50} & \textbf{50} & \textbf{50} & 307 & \textbf{301} & 295 & -4059 & \textbf{-5349} & \textbf{-6387}\\
\hline
Naive & \textbf{50} & \textbf{50} & \textbf{50} & 306 & 300 & 295 & -6397 & -7534 & -8127\\
\hline
Erik & \textbf{50} & 46 & 45 & 306 & 300 & \textbf{296} & -6978 & -7956 & -8474\\
\hline
Derman & 17 & 11 & 9 & 303 & 294 & 282 & \textbf{-3976} & -5450 & -7197\\
\hline
RSVF & \textbf{50} & 46 & 45 & 266 & 262 & 258 & -7465 & -8262 & -8722\\
\hline
BCR & 45 & 43 & 36 & 293 & 288 & 284 & -7400 & -8212 & -8650\\
\hline
RSVI & 45 & 43 & 36 & 291 & 285 & 281 & -7215 & -8087 & -8560\\
\hline
Chow & 30 & 26 & 23 & \textbf{308} & 300 & 289 & -6131 & -6822 & -7489 \\
\hline
\end{tabular}
\caption{Risk of Return for 10,000 Monte Carlo instances}
  \label{tab:appendix_Results}
\end{table}

\section{Additional Related Work} \label{sec:addit-relat-work}

\Cref{tab:related2} summarizes soft-robust and risk-averse results studied in the MDP/RL literature, together with the properties of their proposed formulations and algorithms. Other than the two RASR results presented in this paper: RASR-ERM and RASR-EVaR, we used the name of a representative author to refer to all results in each category. Most relevant refereneces for each entry are as follows: \emph{Iyengar et al.}~\cite{Iyengar2005,Mankowitz2019}, \emph{Xu et al.}~\cite{Xu2012,Wiesemann2013,Grand-Clement2021b}, \emph{Eriksson et al.}~\cite{Eriksson2020}, \emph{Delage et al.}~\cite{Delage2009,Russel2019beyond,Behzadian2021}, \emph{Lobo et al.}~\cite{Lobo2021,Angelotti2021,Javed2021,Brown2020}, \emph{Derman et al.}~\cite{Derman2018}, \emph{Steimle et al.}~\cite{Buchholz2019,Steimle2021a}, \emph{Chen et al.}~\cite{Chen2012}, \emph{Chow et al.}~\cite{Chow2015}, \emph{Osogami et al.}~\cite{Osogami2011}, \emph{Borkar et al.}~\cite{Borkar2002a}.

\begin{table*}
  \centering
  \begin{small}
  \begin{tabular}{l|l|l|l|l|l}
    \toprule
  & & & \multicolumn{2}{|c|}{Risk Measure} & \\    
    Name / author & Horizon & Uncertainty & Epistemic & Aleatory & Complexity \\
    \midrule
    RASR-ERM & Discounted $\infty$ & Dynamic & ERM & ERM & P \\
    RASR-EVaR & Discounted $\infty$ & Dynamic & EVaR & EVaR & P \\
    \midrule
    Iyengar et al. & Discounted $\infty$ & Dynamic & Min & E & P \\
    Xu et al. & Discounted $\infty$ & Dynamic & CVaR & E & NP-Hard \\
    Eriksson et al. & Discounted $\infty$ & Dynamic & ERM & E & -- \\
    Delage et al. & Discounted $\infty$ & Static & VaR & E & NP-Hard \\
    Lobo et al. & Discounted $\infty$& Static & CVaR & E & NP-Hard \\
    Derman et al. & Average $\infty$& Dynamic & E & E & P \\
    Steimle et al. & Finite & Static & E & E & NP-Hard \\
    Chen et al. & Finite & Static & CVaR & E &  NP-Hard\\
    \midrule
    Chow et al. & Discounted $\infty$ & -- & -- & CVaR & NP-Hard \\
    Osogami et al. & Discounted $\infty$ & -- & -- & I-CVaR/I-ERM & P \\
    Borkar et al. & Average $\infty$ & -- & -- & ERM &  \\
    \bottomrule
  \end{tabular}
  \end{small}
  \caption{Summary of the soft-robust and risk-averse models in the MDP/RL literature.}
  \label{tab:related2}
\end{table*}

The description of the rest of the columns is as follows: ``horizon'' indicates the considered MDP setting; ``uncertainty'' shows whether the uncertainty is static or dynamic as discussed in \cref{sec:rasr-erm}; ``risk measure'' contains the risk measure used by the work for epistemic and aleatory uncertainties (with E being the expectation or risk-neutral), and finally, ``complexity'' indicates the complexity of the proposed algorithm(s), if known. \cref{alg:rasr-vi-inf,alg:ERM_EVAR} are marked as ``P'' because they can compute an $\epsilon$-optimal policy in polynomial time for any fixed $\epsilon >0$, $\gamma < 1$, $r_{\min}$, and $r_{\max}$ as shown in \cref{thm:approx-error,thm:rasr-evar-bound}.

\end{document}